\documentclass{article}

\PassOptionsToPackage{numbers, compress}{natbib}

 \usepackage[preprint]{neurips_2026}

\usepackage{bm}
\usepackage{amsmath}
\usepackage{amssymb}
\usepackage{amsthm}

\theoremstyle{remark}
\newtheorem{remark}{Remark}

\theoremstyle{plain}
\newtheorem{theorem}{Theorem}
\newtheorem{proposition}{Proposition}

\newcommand{\bs}{\mathbf{s}}

\newcommand{\dtra}{\mathcal{D}_{\text{tr}}}
\newcommand{\dval}{\mathcal{D}_{\text{v}}}

\newcommand{\tra}{\text{tr}}
\newcommand{\val}{\text{v}}

\newcommand{\bx}{\mathbf{x}}
\newcommand{\bz}{\mathbf{z}}
\newcommand{\bg}{\mathbf{g}}
\newcommand{\be}{\mathbf{e}}
\newcommand{\bl}{\mathbf{l}}

\newcommand{\mcl}[1]{\mathcal{#1}}
\newcommand{\mclb}{\mathcal{B}}
\newcommand{\mcld}{\mathcal{D}}
\newcommand{\mcll}{\mathcal{L}}

\newcommand{\mcls}{\mathcal{S}}

\newcommand{\btheta}{\bm{\theta}}

\usepackage[utf8]{inputenc} 
\usepackage[T1]{fontenc}    
\usepackage{hyperref}       
\usepackage{url}            
\usepackage{booktabs}       
\usepackage{amsfonts}       
\usepackage{nicefrac}       
\usepackage{microtype}      
\usepackage{xcolor}         
\usepackage{hyperref}       
\usepackage{url}            
\usepackage{booktabs}       
\usepackage{amsfonts}       
\usepackage{nicefrac}       
\usepackage{microtype}      
\usepackage{xspace}
\usepackage[utf8]{inputenc} 
\usepackage[T1]{fontenc}    
\usepackage{hyperref}       
\usepackage{url}            
\usepackage{booktabs}       
\usepackage{amsfonts}       
\usepackage{nicefrac}       
\usepackage{microtype}      
\usepackage{xcolor}         
\usepackage{xspace}
\usepackage{bbm}
\usepackage{graphicx}
\usepackage{todonotes} 
\usepackage{subcaption}
\usepackage{enumitem}
\usepackage{hyperref}       
\usepackage{url}            
\usepackage{amsmath}
\usepackage{wrapfig}
\usepackage{booktabs}       
\usepackage{amsfonts}       
\usepackage{nicefrac}       
\usepackage{microtype}      
\usepackage{algorithm}
\usepackage{algorithmic}
\title{Formatting Instructions For NeurIPS 2026}

\title{Online Data Selection for Instruction Tuning via Gaussian Processes}

\author{%
  Jun Wang \hspace{1.5em} Quoc Phong Nguyen \hspace{1.5em} Julien Monteil \hspace{1.5em} Vu Nguyen \\
  Amazon \\
  \texttt{\{junwc, qphong, jul, vutngn\}@amazon.com}
}

\newcommand{\datastyle}[1]{\textsc{#1}\xspace}
\newcommand{\modelstyle}[1]{\textsc{#1}\xspace}

\newcommand{\mmlu}{\datastyle{MMLU}}
\newcommand{\tydiqa}{\datastyle{TyDiQA}} 
\newcommand{\samsum}{\datastyle{SamSum}}
\newcommand{\alpaca}{\datastyle{Alpaca}}
\newcommand{\less}{\datastyle{Less}}
\newcommand{\flan}{\datastyle{Flan V2}}
\newcommand{\datacot}{\datastyle{CoT}}
\newcommand{\dolly}{\datastyle{Dolly}}
\newcommand{\openassistant}{\datastyle{OpenAssistant}}

\newcommand{\qwen}{\modelstyle{Qwen-3-4b}}   
\newcommand{\llama}{\modelstyle{Llama-2-7b}}   
\newcommand{\llamat}{\modelstyle{Llama-3.1-8b}} 
\newcommand{\mistral}{\modelstyle{Mistral-7b}}

\newcommand{\termstyle}[1]{\texttt{#1}}

\newcommand{\gaia}{\textbf{GAIA}}
\newcommand{\greats}{\termstyle{GREATS}}
\newcommand{\maxloss}{\termstyle{MaxLoss}}
\newcommand{\gradnorm}{\termstyle{GradNorm}}
\newcommand{\sbert}{\termstyle{SBERT}}
\newcommand{\rholoss}{\termstyle{RHOLoss}}

\begin{document}

\maketitle

\begin{abstract}
With Large Language Model (LLM) pre-training and fine-tuning shifting its focus from data volume to data quality, quality data selection has emerged as a critical research topic. Existing online data selection methods for LLM training are typically ``batch-constrained'', limiting optimization to local utility within random batches. To overcome this, we propose \textbf{GAIA} (\textbf{G}lobal \textbf{A}daptive \textbf{I}nstruction tuning via G\textbf{A}ussian processes), a framework that formulates data valuation as a global estimation process. GAIA employs Gaussian Process regression to model continuous utility manifolds across the semantic space, utilizing an adaptive strategy fusion mechanism to dynamically prioritize high-utility samples. By casting the strategy-posterior update as an instance of the classical fixed-share Hedge framework for tracking the best expert, we inherit a dynamic-regret guarantee that characterizes GAIA's robustness under non-stationary quality scores during training. Empirical evaluations on three datasets demonstrate that GAIA significantly outperforms state-of-the-art baselines like \greats, establishing our method as a scalable and robust solution for efficient instruction tuning.
\end{abstract}
\section{Introduction}

Modern Large Language Models (LLMs) owe their remarkable success to large-scale pretraining on massive, heterogeneous datasets \cite{achiam2023gpt}. These models are typically adapted to specialized downstream tasks---ranging from dialogue generation to multimodal reasoning---via fine-tuning or instruction alignment \cite{guo2025deepseek, wu2025medreason}. Despite their empirical prowess, LLMs remain susceptible to factual hallucinations, spurious correlations, and ingrained social biases \cite{ferrara2023should}. These vulnerabilities are often traceable to noisy, redundant, or unrepresentative training samples. Consequently, a fundamental challenge in LLM alignment is \emph{data valuation}: identifying which samples meaningfully enhance model generalization and which act as detrimental noise.

Data valuation frameworks, such as influence functions and Shapley value-based methods, provide principled mechanisms for estimating the contribution of individual data points to model performance \cite{ghorbani2019data, koh2017understanding}. While theoretically robust for tasks like mislabeled data detection and bias diagnosis, these classical approaches often struggle with the practical scalability and high-dimensional gradients inherent in modern LLM fine-tuning pipelines. 

To address these scalability constraints, recent works such as \greats{}~\cite{wang2024greats} have pioneered online data selection. However, we observe that existing methods are fundamentally hindered by a \textbf{batch-constrained paradigm}. Specifically, these approaches dynamically evaluate data quality to select an optimal \emph{mini-batch} from a larger \emph{candidate pool}. Crucially, because this candidate pool is still constructed via uniform random sampling, the selection mechanism can only achieve \emph{local optimality}. This design introduces a critical stochastic vulnerability: an unfavorable draw may yield a candidate pool saturated with low-quality samples, forcing the model to choose the ``best of the worst.'' Thus, the efficacy of online selection is inherently upper-bounded by the quality of the stochastic candidate pool, making it limited by the very randomness it seeks to mitigate.

In this paper, we propose \textbf{GAIA} (\textbf{G}lobal \textbf{A}daptive \textbf{I}nstruction Tuning via G\textbf{A}ussian Processes), a novel framework that transcends the batch-constrained limitation by transforming data valuation into a proactive global guidance mechanism. Instead of treating quality scores as passive filters within a batch, \gaia{} leverages them to steer the construction of the candidate pools themselves. Our contributions are three-fold:
\begin{itemize}[leftmargin=*]
    \item We model the latent quality function over the entire training corpus using a \textbf{Gaussian Process (GP)}, enabling principled generalization of valuation signals to unseen examples throughout the training process.
    \item To ensure computational feasibility at scale, we discretize the GP into a finite set of \textbf{strategies} (function samples from the GP prior) and maintain a posterior distribution over them using a fixed-share Hedge update with a top-$k$ likelihood. Casting this update as an instance of the classical framework for tracking the best expert, we inherit a dynamic-regret guarantee that characterizes \gaia{}'s robustness to the non-stationarity of quality scores as model parameters evolve during training.
    \item We demonstrate that \gaia{} can be seamlessly integrated with existing selection methods (e.g., \greats{}). In this hybrid configuration, \gaia{} biases the global sampling toward high-quality regions, while subsequent selection refines the batch for optimal quality and diversity.
\end{itemize}

By bridging theoretical data valuation with practical scalability, \gaia{} provides a principled, computationally efficient framework that explicitly optimizes for downstream generalization, offering a more robust alternative to existing data selection paradigms.
\section{Related Works}

\textbf{Offline Data Valuation.} Most offline methods estimate data importance post-training or via repeated retraining \cite{just2023lava,sava,zhu2025kairos}. While recent works scale Shapley-value \cite{zhangshapley} and influence-function \cite{choe2024your} approaches to LLMs, high computational costs remain. Related fields like dataset distillation and coreset construction \cite{wang2018dataset,mirzasoleiman2020coresets} similarly rely on expensive bilevel optimization. Furthermore, influence-function techniques \cite{koh2017understanding,pruthi2020estimating} assume local linearity around fixed parameters, failing to capture the evolving data utility in the highly non-convex landscapes of LLMs.

The static nature of offline selection---performing pruning only once prior to training---often yields suboptimal results as data informativeness shifts across training stages \cite{wang2024rethinking}. Additionally, many offline pipelines require heavy preprocessing or auxiliary models \cite{xie2023doremi}, increasing system complexity. These limitations necessitate adaptive strategies that evolve alongside the model.

\textbf{Online Data Valuation.} Online methods perform adaptive selection at the batch level, identifying informative examples based on the model's current state. Unlike static approaches \cite{xia2024less,wettigqurating}, this dynamic strategy allows data value to evolve with learning progress, capturing stage-specific relevance for faster convergence and better generalization. By operating on small batches, online selection also reduces preprocessing overhead; for instance, \cite{wang2024greats} uses a greedy Taylor-expansion-based batch optimization.

Recent lightweight proxies use signals like loss, gradient norms, or forgetting events \cite{toneva2019empirical,paul2021deep}. While efficient, these heuristics are often noisy, hyperparameter-sensitive, and lack a unified, task-aware definition of value, making it difficult to align selection with downstream objectives \cite{sorscher2022beyond}.

\section{Preliminaries}


\begin{figure*}
    \centering
    \small
    \includegraphics[width=0.85\linewidth]{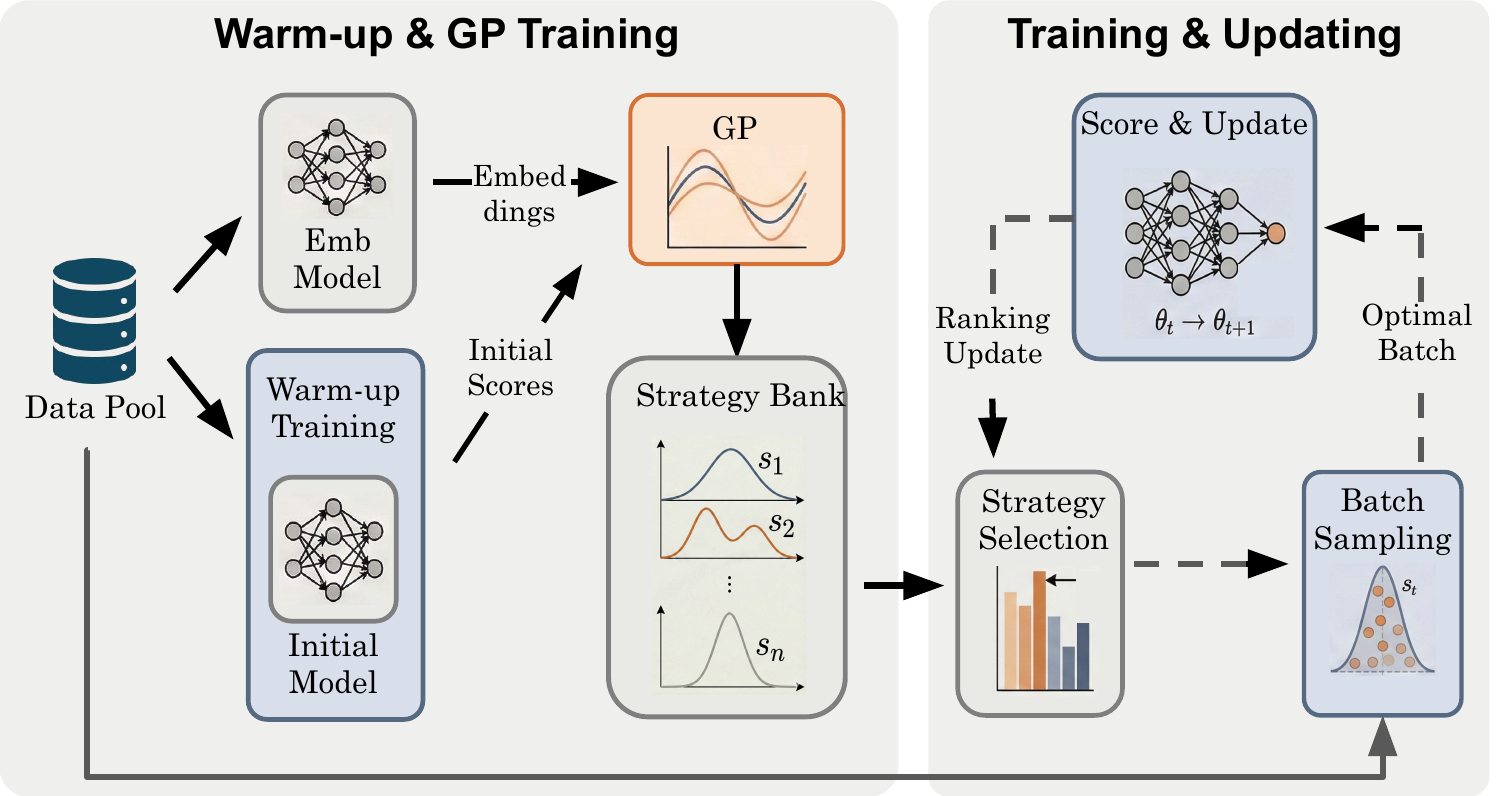}
    \caption{\textbf{Workflow.} Our framework operates in two distinct phases to decouple valuation from training: 
(1) \textbf{Warm-up \& GP Training}: A lightweight \textbf{Emb}eding \textbf{Model} maps the raw data pool into a semantic embedding space. Concurrently, a \textbf{warm-up training} phase on a small dataset set collects initial utility scores. These scores and embeddings are used to fit the Gaussian Process (GP) model, establishing an initial global utility manifold. 
(2) \textbf{Iterative Training \& Adaptive Updating}: In the main loop, an \textbf{adaptive strategy} mechanism selects the optimal strategy from the Strategy Bank based on real-time feedback. This selected strategy guides the sampling of high-utility batches from the global pool for model updates. The resulting feedback is used to dynamically update the strategy selection probabilities, ensuring robust and efficient data selection.}
    \label{fig:enter-label}
    \vspace{-6pt}
\end{figure*}

We use $\dtra, \dval$ for the training and validation sets, respectively, $\bz = (\bx, y)$ for examples, and $\btheta_t$ for model parameters at iteration~$t$. When the source is relevant, we write $\bz^{(\tra)} \in \dtra$ and $\bz^{(\val)} \in \dval$. Training minimizes the empirical loss $\mcll(\dtra; \btheta) \triangleq 1 / |\dtra| \sum_{\bz \in \dtra} l(\bz; \btheta)$ via SGD: at iteration~$t$, a mini-batch $\mclb_t \subset \dtra$ is drawn and $\btheta_{t+1} = \btheta_t - \eta_t \sum_{\bz \in \mclb_t} \bg_t(\bz)$, where $\eta_t > 0$ is the learning rate and $\bg_t(\bz) \triangleq \nabla l(\bz; \btheta_t)$. Full notation is summarized in Table~\ref{tab:notation} (Appendix~\ref{app:notation}).

This work addresses \emph{online data selection}: adaptively choosing~$\mclb_t$ at each iteration. We first review score-based online selection methods, then motivate the need for a surrogate model over training-set quality scores. In Section~\ref{sec:approach}, we instantiate this surrogate as a Gaussian process, which enables principled generalization of observed quality scores to unseen examples. A self-contained background, including the kernel definition and the cost of exact inference, is provided in Appendix~\ref{app:gp}.

Given a randomly sampled candidate batch~$\tilde{\mclb}_t \subset \dtra$ at iteration~$t$, online data selection methods select a refined training batch~$\mclb_t \subseteq \tilde{\mclb}_t$ (used for the SGD update) by evaluating a quality score~$q(\bz)$ for each~$\bz \in \tilde{\mclb}_t$. Proposed scores include gradient alignment with validation data~\cite{wang2024greats}, gradient norms, per-sample loss, and embedding similarity (see Section~\ref{sec:experiment}), and our framework is agnostic to this choice. As a representative example, \greats{}~\cite{wang2024greats} uses~$q(\bz) \triangleq \bg_t(\bz) \cdot \bg_t(\bz^{(\val)})$ and greedily selects~$\mclb_t$ to balance quality and diversity, computed efficiently via the ghost inner-product technique.

However, such methods sample~$\tilde{\mclb}_t$ uniformly at random, ignoring historical quality information. This motivates our central question: \textit{can online data selection be improved by using observed quality scores to guide the sampling of~$\tilde{\mclb}_t$?}

\begin{remark}
\label{remark:skiprefine}
If $\tilde{\mclb}_t$ is already sampled from high-quality regions of $\dtra$, the refinement step (selecting~$\mclb_t \subseteq \tilde{\mclb}_t$ via quality/diversity criteria, e.g.~\greats{}) may be skipped, using~$\mclb_t = \tilde{\mclb}_t$ directly.
\end{remark}

While computing~$q(\bz)$ per example is cheap, doing so across all of~$\dtra$ every iteration is prohibitive. We address this with a surrogate model that predicts scores without explicit gradient computation.

\section{\textbf{GAIA} framework: Strategy-Guided Online Data Selection} \label{sec:approach}
Our goal is to improve online data selection by leveraging historical quality scores of training data to guide the sampling of candidate batches. Specifically, rather than sampling~$\tilde{\mclb}_t$ uniformly at random as in \greats{}~\cite{wang2024greats}, we use observed quality scores from previous iterations to construct an informed sampling distribution over~$\dtra$. The resulting candidate batch~$\tilde{\mclb}_t$ can then be refined using existing selection methods (e.g., \greats{}) to produce the final training batch~$\mclb_t$.

Recall that at each iteration, we evaluate quality scores for all examples in the training batch. By iteration~$t$, we have accumulated observations~$\mcld_t \triangleq \{(\bz, q(\bz))\}_{\bz \in \bigcup_{t'=1}^{t-1} \mclb_{t'}}$. Since~$\mcld_t$ covers only a subset of~$\dtra$, we cannot directly evaluate~$q(\bz)$ for arbitrary training examples.

\subsection{Strategy-Guided Batch Sampling}\label{sec:batch_sampling}
To address this, we model the latent quality function as a random draw from a GP:~$s \sim \text{GP}(0, \kappa)$ where~$\kappa$ denotes the SE kernel (see Appendix~\ref{app:gp}).
We refer to each function sample~$s$ as a \emph{strategy}, as it provides a complete scoring over~$\dtra$ for candidate batch sampling.
Given a strategy~$s$, we sample each element of~$\tilde{\mclb}_t$ independently from~$\dtra$ according to the categorical distribution
\begin{align}
    p(\bz \mid s) \triangleq \exp(s(\bz) / \tau) \Big/ \sum\nolimits_{\bz' \in \dtra} \exp(s(\bz') / \tau)\ ,
\end{align}
where~$\tau > 0$ controls the sampling temperature.

\begin{remark}
\label{remark:adaptive_temp}
The temperature $\tau$ governs the exploration-exploitation trade-off: lower values concentrate sampling on high-scoring examples, while higher values encourage broader exploration. In practice, $\tau$ can be adapted during training, e.g., gradually increasing to promote exploration as the model converges (see Section~\ref{sec:experiment}).
\end{remark}

A key challenge remains: how to efficiently obtain~$s$ from the GP posterior at each training iteration.
As noted in Appendix~\ref{app:gp}, updating the GP with new observations at every iteration is computationally prohibitive.

\subsection{Efficient Strategy Generation via GP Discretization}

To enable efficient posterior updates and strategy sampling, we discretize the GP by drawing a finite set of function samples from the prior.
Specifically, we sample $m/2$ functions $f_1, \ldots, f_{m/2}$ from $\text{GP}(0, \kappa)$ and set $\mcls \triangleq \{f_1, -f_1, \ldots, f_{m/2}, -f_{m/2}\}$, so that for every strategy $f \in \mcls$ its negation $-f$ is also in $\mcls$.
Since these functions are drawn from the prior, we initialize with a uniform distribution: $p(s) = 1/|\mcls|$ for all $s \in \mcls$.
This discretization preserves the key property encoded by the kernel: similar training examples have similar latent quality scores under each strategy.

\begin{remark}[Symmetry of the strategy set]\label{remark:symmetric_strategies}
The pairing $\{f, -f\}$ ensures that the strategy set has no directional bias along any sampled function: for any two points $\bz, \bz' \in \dtra$, the pair ranks $\bz$ above $\bz'$ under $f$ if and only if it ranks $\bz$ below $\bz'$ under $-f$, with equal weight. The pair thus averages out directional preferences, producing a joint marginal $\frac{1}{2}[p(\bz \mid f) + p(\bz \mid -f)]$ that is strictly closer to uniform than either $p(\cdot \mid f)$ or $p(\cdot \mid -f)$ individually. This property supports the interpretation of the uniform mixing step introduced below as a return toward unbiased exploration: when the mixing weight dominates, the algorithm samples from an ensemble that treats high-score and low-score regions symmetrically.
\end{remark}

\textbf{Motivation for a non-stationary posterior update.} A standard Bayesian posterior update over $\mcls$ presumes that observations are generated by a single time-invariant strategy. This assumption does not hold in our setting. The observed quality score $q(\bz)$ depends on the current model parameters $\btheta_t$, which change throughout training, so the strategy that best explains recent batches can differ from the one that best explained earlier batches. Under a standard Bayesian update, the posterior concentrates on whichever strategy accumulated the most evidence in the past and cannot recover once a different strategy becomes more informative. To address this, we adopt the \emph{fixed-share Hedge} update of \citet{herbster1998tracking}, which tempers each batch's evidence by a learning rate $\lambda$ and re-mixes the posterior with a uniform distribution at every step. The uniform mixing ensures a strictly positive floor on every strategy's probability, so no strategy is permanently ruled out. This update admits an exact Bayesian interpretation as the filtering distribution under a Markovian switching prior over strategies.

Given observations $\mcld_t$, the resulting update reduces to an $\mathcal{O}(|\mcls|)$ incremental two-step procedure:
\begin{align}
    \forall s \in \mcl{S},\ \hat{p}(s|\mcld_t) \propto p(s|\mcld_{t-1}) \, p \bigl(q(\mclb_t) | s \bigr)^\lambda\ ,\quad p(s|\mcld_t) = \frac{\alpha}{m} + (1 - \alpha)\,\hat{p}(s|\mcld_t)\ ,
\end{align}
where $q(\mclb_t) \triangleq \{q(\bz)\}_{\bz \in \mclb_{t}}$, $\lambda > 0$ is a learning rate controlling how strongly each batch's evidence shifts the posterior, and $\alpha \in [0,1)$ is the \emph{mixing rate}. The uniform floor $\alpha/m$ propagates into the next iteration's Hedge step, allowing a previously down-weighted strategy to regain mass once it starts producing high-likelihood batches. Setting $\alpha = 0$ recovers pure Hedge, which corresponds to a stationary Bayesian posterior. Section~\ref{sec:regret} states the dynamic-regret guarantee that this update inherits from fixed-share Hedge.
Under the Gaussian likelihood in Equation ~\eqref{eq:gausslik}, the likelihood factorizes as
\begin{align}
    p_{\text{Gauss}} \bigl(q(\mclb_t) | s \bigr) \triangleq \prod_{\bz \in \mclb_t} \mathcal{N} \bigl(q(\bz); s(\bz), \sigma_n^2 \bigr)\ ,
\end{align}
where $p_{\text{Gauss}}$ highlights the use of the Gaussian likelihood.

\begin{algorithm}[t]
\caption{Strategy-Guided Online Data Selection}
\label{alg:strategy-selection}
\begin{algorithmic}[1]
    \REQUIRE $\dtra$, $\dval$, \#strategies $m$, Hedge learning rate $\lambda$, mixing rate $\alpha$, temperatures $\tau, \beta$, $k$ , target model parameters $\btheta$
\STATE \textbf{Initialization:}
\STATE \quad Obtain embeddings $\{\be(\bz)\}_{\bz \in \dtra}$
\STATE \quad Fit GP kernel hyperparameters on warm-up data
\STATE \quad Sample strategies $\mcls = \{f_1, -f_1, \ldots, f_{m/2}, -f_{m/2}\}$ with $f_j \sim \text{GP}(0, \kappa)$
\STATE \quad Initialize $p(s) \gets 1/m$ for all $s \in \mcls$
\FOR{$t = 1, 2, \ldots, T$}
    \STATE \textbf{Strategy Selection:}
    \STATE \quad Sample strategy $s \sim p(\cdot | \mcld_{t-1})$
    \STATE \textbf{Candidate Batch Sampling} (see Remark~\ref{remark:adaptive_temp})\textbf{:}
    \STATE \quad Sample $\tilde{\mclb}_t \subset \dtra$ where each $\bz \in \tilde{\mclb}_t$ is drawn with $p(\bz|s) \propto \exp(s(\bz)/\tau)$
    \STATE \textbf{Batch Refinement} (optional, see Remark~\ref{remark:skiprefine})\textbf{:}
    \STATE \quad Compute $\{q(\bz)\}_{\bz \in \tilde{\mclb}_t}$ using $\dval$
    \STATE \quad Refine $\tilde{\mclb}_t$ to $\mclb_t \subseteq \tilde{\mclb}_t$
    \STATE \textbf{Training:}
    \STATE \quad Update $\btheta_t \to \btheta_{t+1}$ using $\mclb_t$
    \STATE \textbf{Strategy Posterior Update (fixed-share Hedge):}
    \STATE \quad $\forall s \in \mcls$: $\hat{p}(s|\mcld_t) \propto p(s|\mcld_{t-1}) \times p_{\text{top-}k}(q(\mclb_t) | s)^\lambda$
    \STATE \quad $\forall s \in \mcls$: $p(s|\mcld_t) \gets \alpha/m + (1 - \alpha)\, \hat{p}(s|\mcld_t)$
\ENDFOR
\STATE \textbf{return} $\btheta_T$
\end{algorithmic}
\end{algorithm}

\subsection{Top-\texorpdfstring{$k$}{k} Likelihood}

While the Gaussian likelihood in Equation ~\eqref{eq:gausslik} is commonly employed in GP modeling, it is not well-suited to our setting: the range of observed quality scores can vary drastically across iterations as the model parameters evolve.
To address this, we propose a top-$k$ likelihood that depends only on the relative ordering of quality scores within each batch.

Let $\bz_{(1)}, \ldots, \bz_{(|\mclb_t|)}$ denote the examples in $\mclb_t$ sorted by observed quality score in descending order, i.e., $q(\bz_{(1)}) \geq q(\bz_{(2)}) \geq \cdots \geq q(\bz_{(|\mclb_t|)})$.
We define the top-$k$ likelihood as the probability that strategy $s$ ranks the same $k$ examples highest:
%
\begin{align}
    p_{\text{top-}k} \bigl(q(\mclb_t) | s \bigr) \triangleq \prod\nolimits_{j=1}^{k} \left[ \exp(s(\bz_{(j)}) / \beta) \Big/ \sum\nolimits_{i=1}^{|\mclb_t|} \exp(s(\bz_{(i)}) / \beta) \right]\ ,
\end{align}
where $\beta > 0$ is a temperature parameter. 
Algorithm~\ref{alg:strategy-selection} summarizes the framework for \gaia{}.

\subsection{Dynamic-Regret Guarantee}\label{sec:regret}

The posterior update in Section~\ref{sec:batch_sampling} is designed to be robust to non-stationarity, and fixed-share Hedge is the canonical online-learning algorithm with this property. This subsection makes the robustness precise. For a sequence of loss functions $\ell_t(s) \triangleq -\log p_{\text{top-}k}(q(\mclb_t) | s)$, the natural non-stationary baseline is a comparator \emph{sequence} $\bs^* = (s_1^*, \ldots, s_T^*)$ that may use a different strategy at each iteration, subject to switching between strategies at most $K$ times. Fixed-share Hedge is known to track such a comparator with sublinear regret.

\begin{theorem}[Dynamic regret; cf.\ \citealp{herbster1998tracking,cesa2006prediction}]
\label{thm:regret-main}
As shown in Appendix~\ref{app:regret}, the update in Section~\ref{sec:batch_sampling} is exactly fixed-share Hedge applied to $\ell_t$ with learning rate $\lambda$ and mixing rate $\alpha$. Let $L$ be a uniform bound on $\ell_t$ over $\mcls$. For any comparator $\bs^* \in \mcls^T$ with at most $K$ switches, there exist choices of $\lambda$ and $\alpha$ such that
\begin{align}
    \sum_{t=1}^T \mathbb{E}_{s \sim p_t}[\ell_t(s)] - \sum_{t=1}^T \ell_t(s_t^*) \leq \mathcal{O}\Bigl( L \sqrt{T \bigl( K \log m + K \log(T/K) + \log m \bigr)} \Bigr)\ .
\end{align}
\end{theorem}

The bound is sublinear in $T$ for any $K$ sublinear in $T$ (i.e., $K \le \sqrt{T}$), so the algorithm remains \emph{no-regret} as long as the optimal strategy does not switch too frequently. This justifies the algorithmic form of our posterior update as a response to non-stationarity: it is the mechanism by which the algorithm avoids committing irreversibly to a strategy that may become suboptimal later.

\textbf{Interpretation.} The bound is $\mathcal{O}\bigl(L \sqrt{T (K \log m + K \log(T / K) + \log m)}\bigr)$ in $T$, sublinear whenever $K$ is sublinear in $T$, and it decomposes into three interpretable cost terms. The factor $\sqrt{T \log m}$ captures the cost of discretizing the GP into $m$ strategies, which is the classical Hedge penalty. The factor $\sqrt{T K \log(T / K)}$ is the additional cost of tracking a comparator that switches up to $K$ times. The loss bound $L = k(\log b + 2 S / \beta)$ ties the bound to \gaia{}'s design parameters: the top-$k$ truncation $k$ enters linearly, the likelihood temperature $\beta$ appears via $1 / \beta$ (trading the regret bound against the sharpness of the likelihood signal), and the scale $S$ is the maximum absolute value of strategy scores on the training set, a finite constant determined by the sampled strategies. Setting $K = 0$ recovers the standard static-regret bound $\mathcal{O}\bigl(L \sqrt{T \log m}\bigr)$ with $\alpha^* = 0$, which corresponds to pure Hedge, and the regret grows with $K$ only through the slowly growing factor $\sqrt{K \log(T / K)}$, so the algorithm remains no-regret as long as the number of switches is $o(T / \log m)$.

\textbf{Practical remarks on tuning $\lambda$ and $\alpha$.} The optimal $\lambda^*$ and $\alpha^*$ in Theorem~\ref{thm:regret-main} depend on the horizon $T$, the loss bound $L$, and the number of switches $K$. Since $K$ is not observable and characterizes how much the identity of the optimal strategy changes as training progresses, setting $\lambda$ and $\alpha$ from the theoretical expressions requires either an estimate of $K$ or an adaptive procedure such as the doubling trick over candidate values of $K$, which incurs an additional logarithmic factor. In our experiments, we bypass this estimation problem and tune $\lambda$ and $\alpha$ empirically on a small grid, using the theoretical bound only as qualitative guidance (small $\alpha$ for sharp exploitation, larger $\alpha$ for broader tracking).

\section{Experiments}
\label{sec:experiment}
We follow the settings of \cite{wang2024greats}. We evaluated our method on three instruction tuning tasks, \mmlu{}~\cite{hendryckstest2021}, \samsum{}~\cite{gliwa-etal-2019-samsum} and \tydiqa{}~\cite{tydiqa}, and on four backbone models, \llama{}~\cite{touvron2023llama}, \llamat{}~\cite{grattafiori2024llama}, \qwen{}~\cite{yang2025qwen3} and \mistral{}~\cite{mistrary-7b}. Training details are presented in Appendix~\ref{app:tasks}.
We evaluate \gaia{} against a diverse set of representative methods that also serve as our score functions, including gradient-based (\gradnorm), loss-prioritizing (\maxloss, \rholoss), semantic-aware (\sbert), and greedy optimization (\greats) methods; detailed descriptions are provided in Appendix~\ref{app:baselines}.

\begin{figure*}[t]
    \centering
    \begin{subfigure}{\linewidth}
        \centering
        \includegraphics[width=0.99\linewidth]{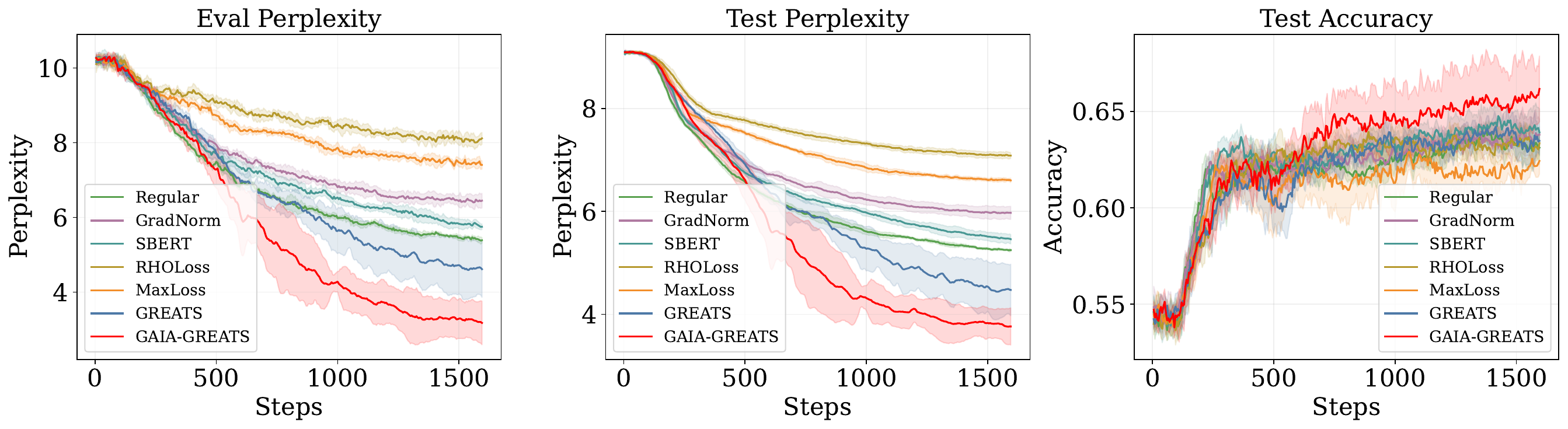}
        \caption{MMLU - Sociology: Validation Perplexity, Test Perplexity, and Accuracy (from left to right)}
        \label{fig:res_sociology}
    \end{subfigure}


    \begin{subfigure}{\linewidth}
        \centering
        \includegraphics[width=0.99\linewidth]{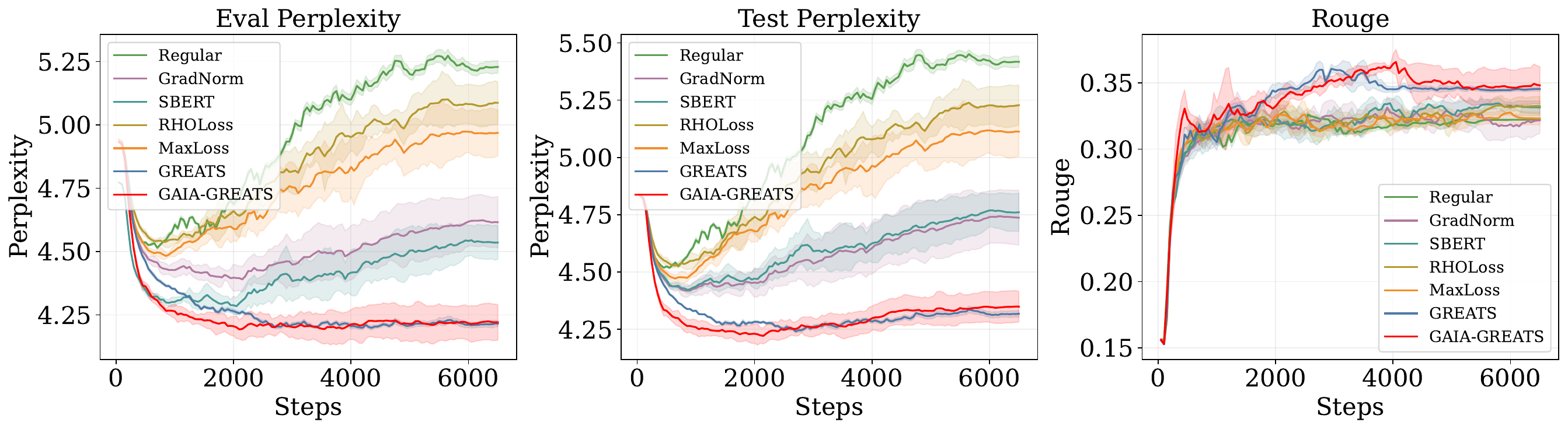}
        \caption{SAMSum: Validation Perplexity, Test Perplexity, and Rouge (from left to right)}
        \label{fig:res_samsum}
    \end{subfigure}

    \caption{Training dynamics on \mmlu-Sociology (top row) and \samsum (bottom row). We benchmark \gaia{}-\greats{} against full-data training (\textit{Regular}) and various selection baselines. The left and middle columns report validation and test perplexity (lower is better), while the right column presents downstream task performance (Accuracy for \mmlu, ROUGE for \samsum). \gaia{}-\greats{} demonstrates superior convergence speed and robust generalization, consistently achieving the lowest perplexity levels and highest task scores across both datasets. All results are reported as the average of three independent runs.}
    \label{fig:main_results}
    \vspace{-8pt}
\end{figure*}

\subsection{Efficient Gaussian Process Construction}
To make GP inference computationally feasible for LLMs, we introduce a decoupled pipeline.

\textbf{Proxy Semantic Space.} Rather than using evolving hidden states, we extract static features $\be(\bx)$ via a sentence transformer \cite{reimers-2019-sentence-bert} \textbf{(robustness evaluated in Appendix \ref{app:embedding_methods})}. To further accelerate inference, SVD projects these embeddings into a lower-dimensional manifold \textbf{(dimensionality analysis in Appendix \ref{app:embedding_dim})}.

\textbf{Warm-up Initialization.} To address label scarcity, we evaluate true utility scores on a small subset ($N=500$) during early training \textbf{(size sensitivity analyzed in Appendix \ref{app:warmup})}. The GP is then fitted to these observations to generalize predictions across the remaining data pool. Implementation details are in Appendix \ref{app:implementation}.

\begin{figure*}[t]
   \vspace{-1pt}
    \centering

    \begin{subfigure}{0.245\textwidth}
        \centering
        \includegraphics[width=\linewidth]{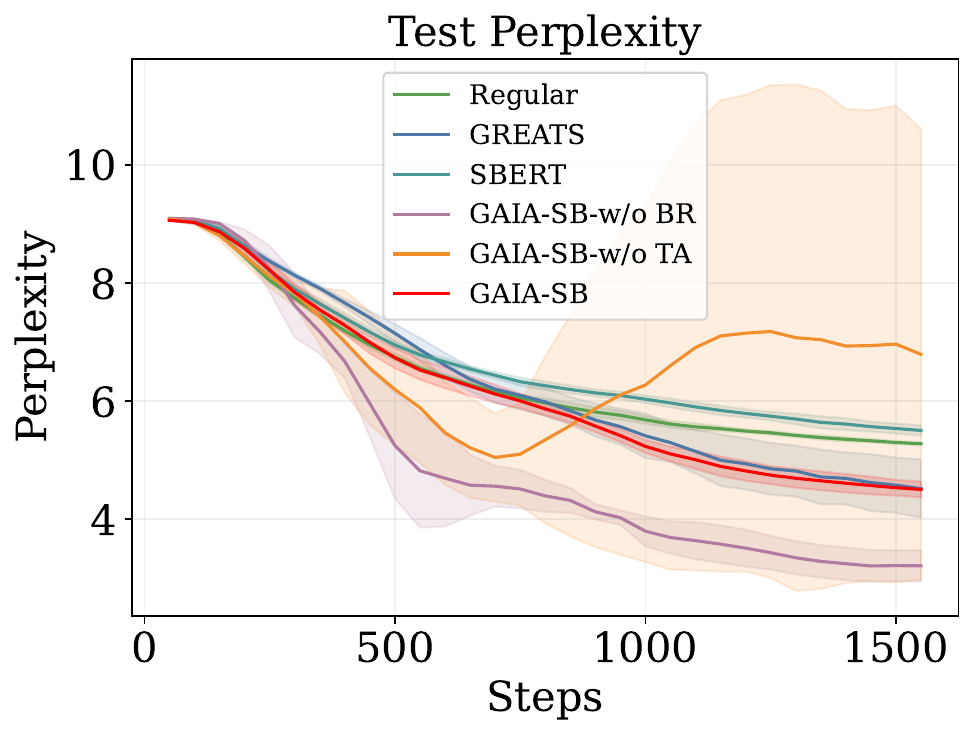}
        \caption{\gaia{}-\sbert}
        \label{fig:res_sociology-s}
    \end{subfigure}
    \hfill
    \begin{subfigure}{0.245\textwidth}
        \centering
        \includegraphics[width=\linewidth]{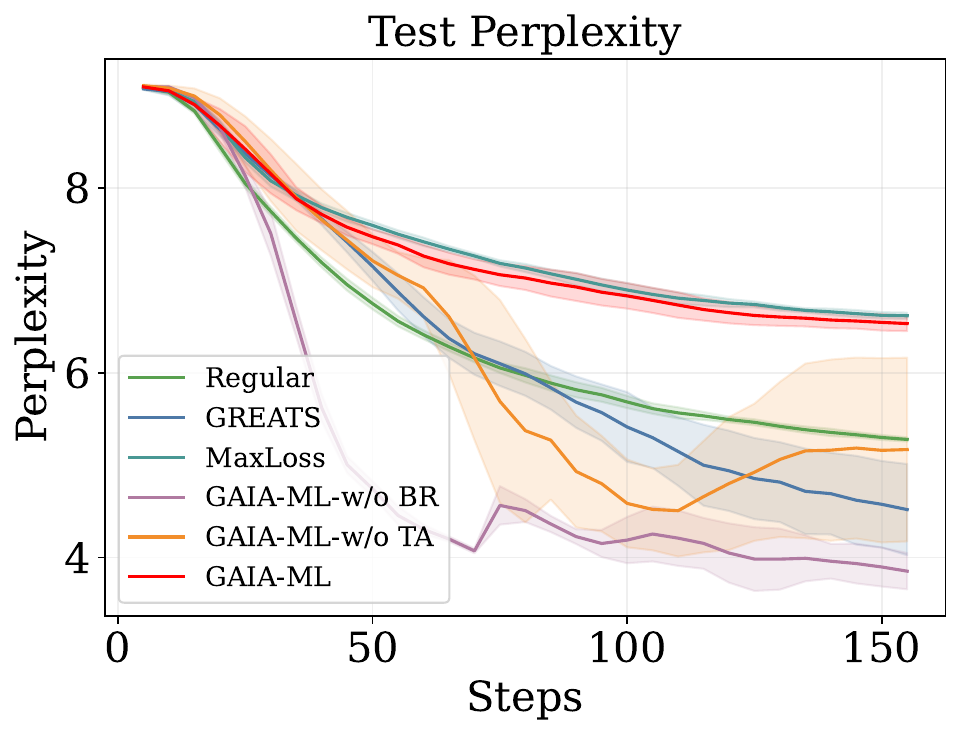}
        \caption{\gaia{}-\maxloss}
        \label{fig:res_samsum-s}
    \end{subfigure}
    \hfill
    \begin{subfigure}{0.245\textwidth}
        \centering
        \includegraphics[width=\linewidth]{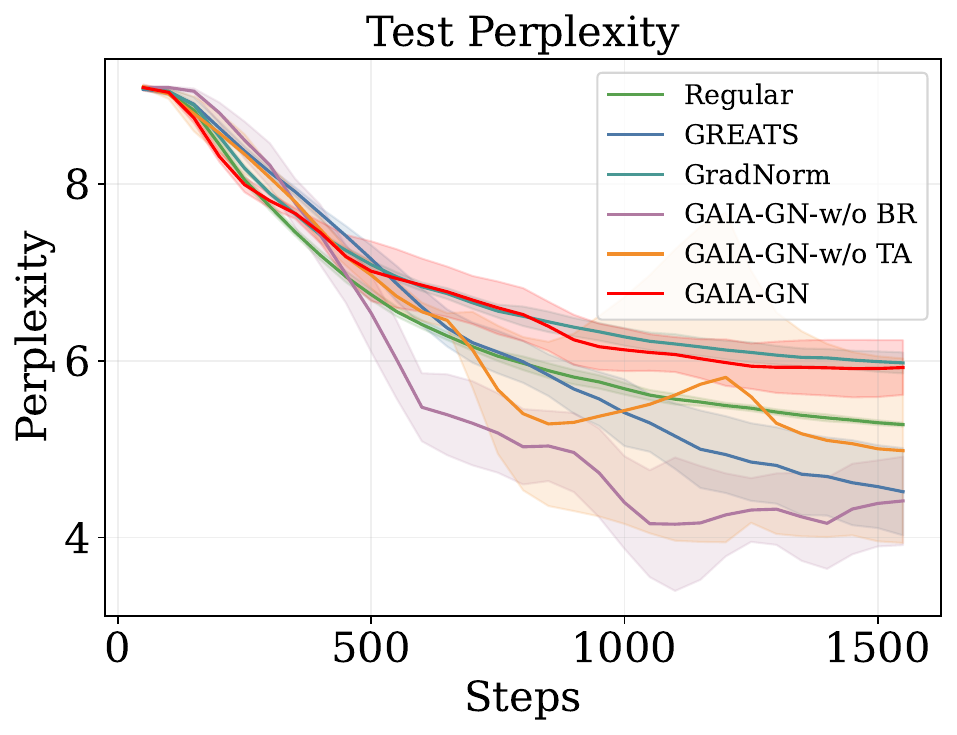}
        \caption{\gaia{}-\gradnorm}
        \label{fig:res_data3}
    \end{subfigure}
    \hfill
    \begin{subfigure}{0.245\textwidth}
        \centering
        \includegraphics[width=\linewidth]{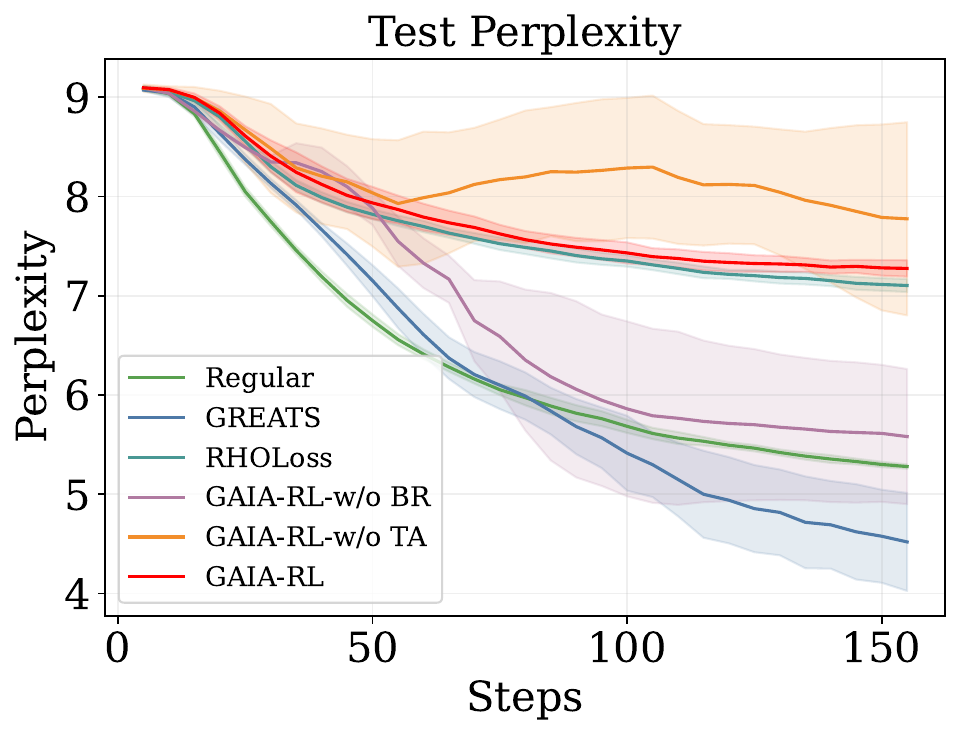}
        \caption{\gaia{}-\rholoss}
        \label{fig:res_data4}
    \end{subfigure}

   \vspace{-2pt}
    \caption{Performance comparison of four different scoring functions on \mmlu-Sociology. All results are reported as the average of three independent runs. The annotation \textbf{w/o BR} denotes the exclusion of the \textit{Batch Refinement} step, while \textbf{w/o TA} indicates the removal of the \textit{Temperature Adaptation} mechanism.}
    \label{fig:scoring}
    \vspace{-6pt}
\end{figure*}

\subsection{Results and Analysis}

Figure \ref{fig:main_results} depicts the comparative training dynamics of \gaia{}-\greats{} (where our \gaia{} framework leverages the \greats{} scoring function) against baselines on \mmlu-Sociology and \samsum{}. Additional MMLU subjects, \textsc{TyDiQA} and Qwen model results are reported in Appendix~\ref{app:mmlu_addition}, Appendix~\ref{app:tydiqa} and Appendix~\ref{app:qwen}, respectively. Notably, these performance gains come with negligible computational overhead: for instance, on \mmlu{} ($\approx$13k samples), it adds only $\approx$23 seconds for GP training. A detailed efficiency analysis is provided in Appendix~\ref{app:time}.

\textbf{Training Dynamics.} As shown in the perplexity plots, \gaia{}-\greats{} (red line) exhibits a significantly faster convergence rate than the baselines on both evaluation and test sets. Among the competitors, \greats{} proves to be the strongest baseline; however, our method demonstrates distinct advantages depending on the task. On \mmlu-Sociology, \gaia{}-\greats{} consistently maintains a lower perplexity throughout training, achieving a minimum of \textbf{3.75$\pm$0.33} compared to 4.37$\pm$0.64 for \greats{}, representing a substantial \textbf{14.13\% reduction}. This indicates a more precise estimation of data utility. To illustrate this behavior, we visualize the specific GP data sampling dynamics for the \mmlu-Sociology task in Appendix~\ref{app:sampling_dynamics}. On the \samsum{} task, while both methods eventually reach a comparable performance floor, \gaia{}-\greats{} distinguishes itself with a much sharper descent in the initial phase. This rapid drop indicates that our Gaussian Process model effectively identifies high-value data points immediately at the start, significantly accelerating the early learning process compared to the gradient-based selection in \greats{}.

\textbf{Downstream Performance.} Our method consistently outperforms or matches the best competitors. For \mmlu-Sociology, \gaia{}-\greats{} achieves a peak accuracy of \textbf{66.50\%$\pm$1.92}, surpassing the second-best method \greats{} (64.34\% $\pm$1.02) with a clear gain of \textbf{+3.57\%}. On the \samsum{} summarization task, we evaluate performance using the ROUGE score. While standard methods suffer from overfitting, evidenced by increasing test perplexity after 2000 steps, both \greats{} and our method maintain robustness. Notably, although the final ROUGE scores of \greats{} and our method are similar, \gaia{}-\greats{} achieves higher ROUGE scores in the early training stages. This further validates that our strategy successfully prioritizes the most informative samples early on, leading to a more efficient optimization path without sacrificing generalization capability.

\textbf{Comparison of Cumulative Data Value Scores.}
We analyze the learning dynamics by comparing the cumulative data value scores of data selected by \gaia{} versus uniform random sampling. 
\begin{wrapfigure}{r}{0.45\textwidth} 
    \small
    \centering
    \includegraphics[width=0.7\linewidth]{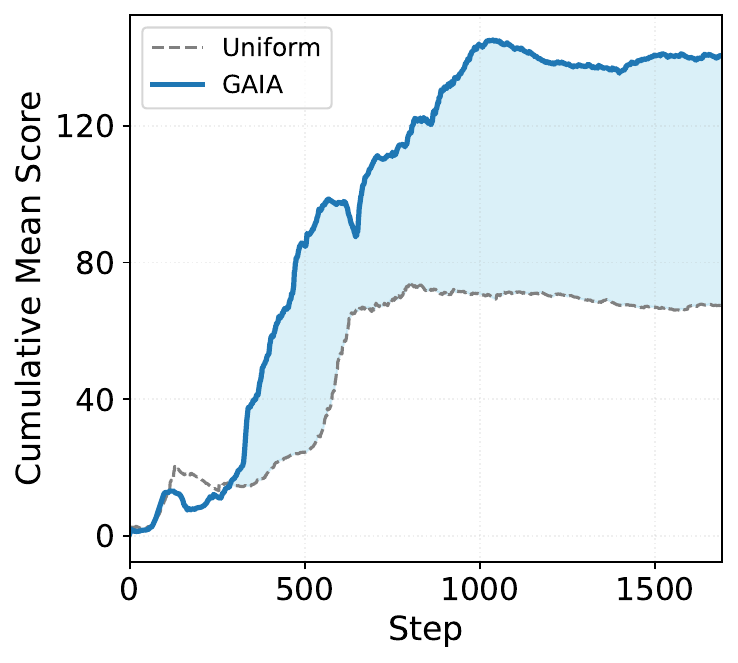}
    \caption{Cumulative data value scores (\llama{}, MMLU-Sociology).}
    \label{fig:scores}
\end{wrapfigure}
As illustrated in Figure~\ref{fig:scores}, \gaia{} consistently prioritizes high-value samples throughout the training process, resulting in a cumulative score that significantly outpaces the baseline. This gap demonstrates that while random sampling inevitably consumes redundant or low-utility data, \gaia{} effectively steers the model toward the most informative regions of the data distribution, maintaining a high selection quality across all training stages.

\textbf{Robustness across Scoring Functions.}
To assess the universality of our framework, we integrate our GP-based sampling mechanism with diverse scoring functions, including \sbert, \gradnorm, \maxloss, and \rholoss. The results in Figure \ref{fig:scoring} demonstrate that our method functions as a robust performance amplifier dependent on signal fidelity. For high-quality metrics like \sbert, our framework significantly boosts convergence, achieving results comparable to the strong \greats{} baseline. Conversely, for weaker metrics (\gradnorm, \maxloss, \rholoss)—where standard selection strategies often fail to outperform regular training—our method yields limited additional gains. \textbf{However, it consistently maintains performance parity or achieves slight improvements over the baselines.} This confirms that our framework acts as a safe enhancement: it effectively exploits informative signals to accelerate learning while ensuring that the training stability is not compromised, even when the underlying utility estimation is noisy.

\textbf{Efficacy of Global Selection.}
To isolate the intrinsic contribution of our global selection mechanism, we conducted an ablation study by disabling the \textit{batch refinement} step. In this \textbf{w/o BR} configuration, the model is trained directly on the raw batch sampled via the GP framework, bypassing any local sub-selection. Remarkably, as illustrated in Figure~\ref{fig:scoring}, this simplified approach achieves performance superior to the full pipeline. This finding provides compelling empirical evidence for the inherent effectiveness of our framework: it demonstrates that the data points identified through global utility estimation are of high quality by themselves. Furthermore, it confirms that the primary driver of our method's success is its capacity to locate globally informative samples, rather than a reliance on local, batch-constrained optimization.

\begin{figure}[ht]
   \vspace{-4pt}
    \centering
    \small
    \begin{subfigure}{0.30\linewidth}
        \centering
        \includegraphics[width=\linewidth]{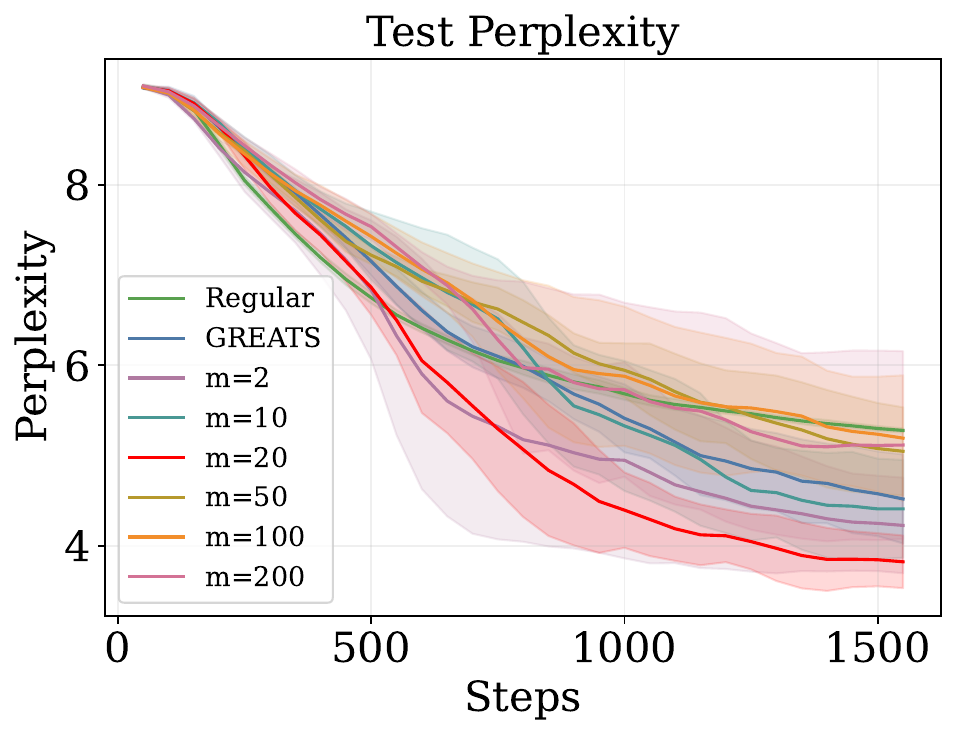}
        \caption{$m$ on \mmlu-Sociology}
        \label{fig:res_sociology_n}
    \end{subfigure}\hfill 
        \begin{subfigure}{0.31\linewidth}
        \centering
        \includegraphics[width=\linewidth]{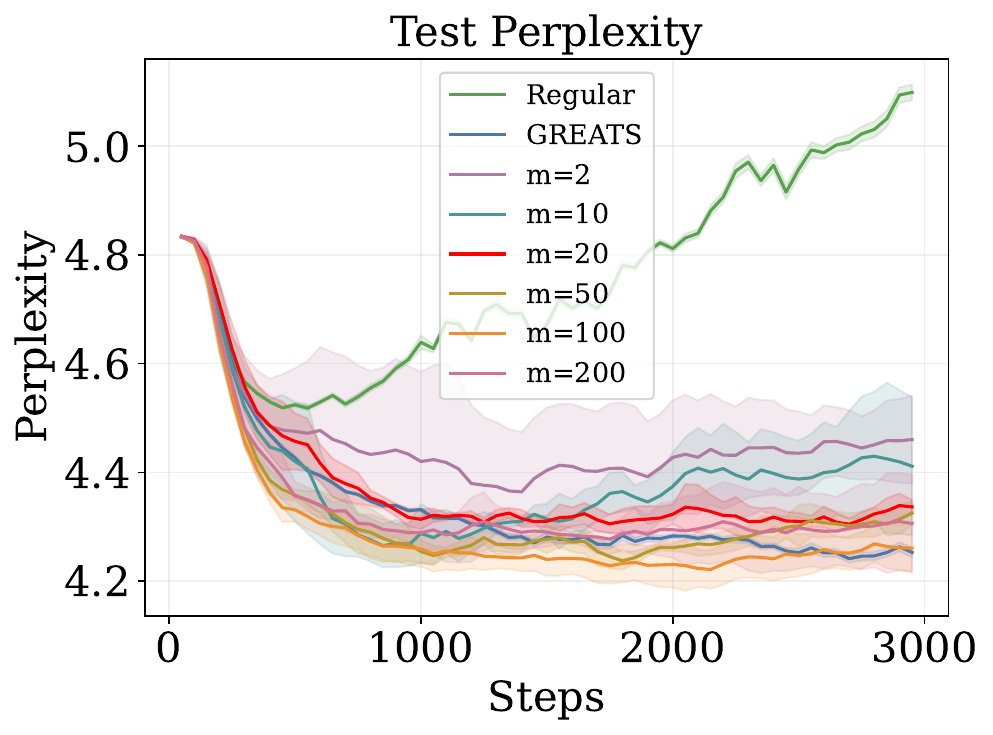}
        \caption{$m$ on \samsum}
        \label{fig:res_sociology_n}
    \end{subfigure}\hfill 
    \begin{subfigure}{0.30\linewidth}
        \centering
        \includegraphics[width=\linewidth]{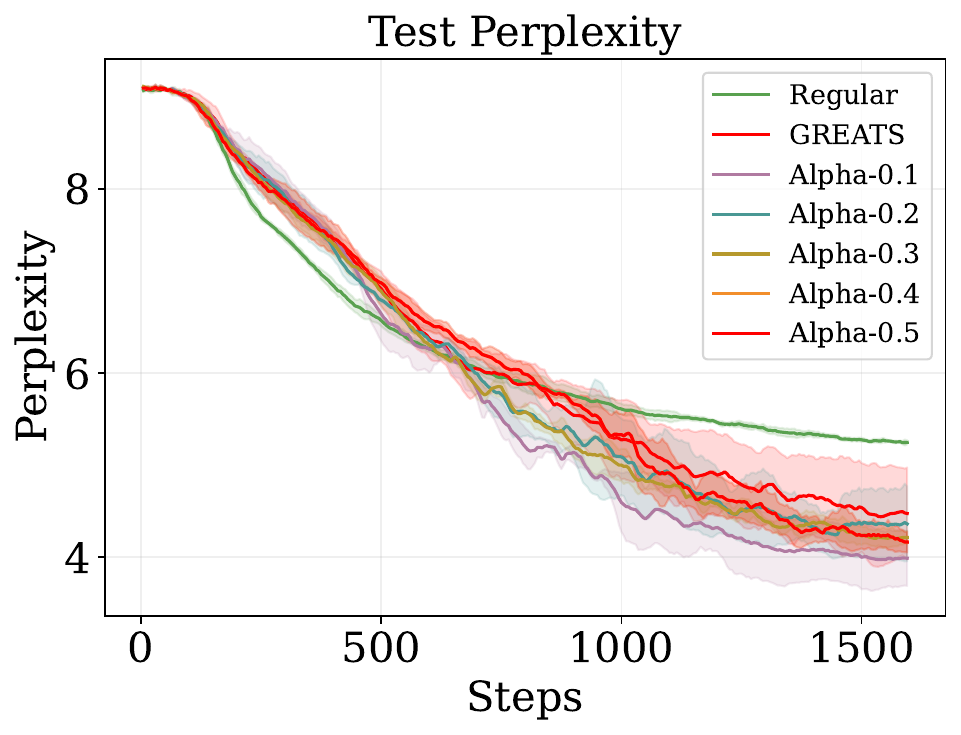}
        \caption{$\alpha$ on \mmlu-Sociology}
        \label{fig:res_samsum-n}
    \end{subfigure}
       \vspace{-3pt}
    \caption{Impact of strategy number on test perplexity for \mmlu-Sociology and \samsum and $\alpha$ for \mmlu-Sociology. We benchmark \gaia{}-\greats~against \textit{Regular} training and the \greats{} baseline. All results are averaged over three independent runs.}
    \label{fig:strategy_impact}
    \vspace{-4pt}
\end{figure}

\textbf{Effect of Strategy Pool Size.}
Results in Figure~\ref{fig:strategy_impact} show that the optimal strategy pool size $m$ varies with the size of the dataset. For the smaller \mmlu-Sociology, a compact pool ($m=20$) is optimal, as excessive strategies tend to dilute the utility signal toward uniform sampling (\textit{Regular} baseline). Conversely, the larger \samsum dataset demands greater strategic diversity; while $m \geq 50$ yields strong convergence, a constrained pool ($m < 50$) results in a noticeably higher perplexity floor due to insufficient utility coverage. Therefore, $m$ must be adaptively scaled: we should prioritize focused selection for small datasets and expand the ensemble for larger datasets, and more complex tasks.

\textbf{Impact of Sampling Temperature.}
The sampling temperature $\tau$ serves as a critical control for the exploration-exploitation trade-off within our GP sampling framework. Our analysis reveals a distinct tension: lower temperatures accelerate early-stage convergence via aggressive exploitation of high-utility samples, whereas higher temperatures ensure batch diversity but slow down the identification of informative data (See Appendix~\ref{app:adaptive_sensitivity} for details experiments). To reconcile these objectives, \gaia{}-\greats{} employs an \textbf{adaptive temperature} mechanism that dynamically transitions from focused exploitation to broader exploration. The necessity of this schedule is empirically confirmed by the \textbf{w/o TA} ablation results (Figure~\ref{fig:scoring}), where using a fixed temperature leads to significant training instability and mid-stage performance corruption. We provide a detailed discussion of the adaptive mechanism in Appendix~\ref{app:adaptive_mech}.

\textbf{Effect of Mixing Rate $\alpha$.} The fixed-share mixing rate $\alpha$ enforces a uniform floor on every strategy's probability, allowing the posterior to recover when a previously down-weighted strategy becomes more informative. Sweeping $\alpha \in \{0.1, 0.2, 0.3, 0.4, 0.5\}$ on \mmlu-Sociology with \llama{}, we find that $\alpha = 0.1$ achieves the lowest validation and test perplexity, and is among the highest test accuracies. This finding supports two complementary design choices. First, non-zero mixing is essential: it prevents the posterior from irreversibly committing to an early-favored strategy, empirically validating the fixed-share update over pure Hedge. Second, the optimal mixing is small, indicating that the GP-informed posterior signal should be largely preserved rather than diluted toward the uniform prior. A sharp posterior anchored to the Gaussian process does the heavy lifting of identifying high-utility regions, while a modest mixing floor maintains robustness under the non-stationarity of quality scores.
%


%
%

\section{Conclusion}

\textbf{Limitations.} Despite its effectiveness, our method has several limitations: first, its performance is tied to the fidelity of the base scoring function, acting as an amplifier rather than a corrective mechanism for suboptimal data utility signals; second, the method exhibits sensitivity to hyperparameters such as the strategy pool size $m$ and sampling temperature $\tau$, which require calibration based on dataset size to effectively balance exploration and exploitation; and finally, our current evaluation focuses exclusively on instruction tuning, leaving the application of global data valuation in large-scale pre-training as future work.

\textbf{Conclusion.} We have presented an adaptive data selection framework for LLM fine-tuning that moves beyond conventional batch-constrained approaches. By modeling global utility distributions with Gaussian Process regression and dynamically fusing multiple valuation strategies, our method ensures consistent delivery of high-quality training batches. Casting the strategy-posterior update as an instance of the classical fixed-share Hedge framework for tracking the best expert, \gaia{} inherits a dynamic-regret guarantee that characterizes its robustness under non-stationary quality scores. Empirical results across \mmlu{}, \tydiqa{}, and \samsum{} benchmarks demonstrate that our approach accelerates convergence and improves final task performance compared to state-of-the-art baselines. Beyond improving efficiency and accuracy, our framework provides a scalable, robust, and principled solution for data-efficient instruction tuning, highlighting the importance of intelligent data selection in large-scale model training.


\bibliography{references}{}
\bibliographystyle{unsrtnat}

\newpage
\appendix
\section{Appendix}

\subsection{Notation}\label{app:notation}

\begin{table}[h]
\centering
\caption{Summary of notation.}
\label{tab:notation}
\small
\begin{tabular}{@{}ll@{}}
\toprule
Symbol                            & Description                              \\
\midrule
$\dtra$,~$\dval$                  & Training and validation sets             \\
$\bz \triangleq (\bx,y)$                     & Training example (input, target)         \\
$\btheta_t$                       & Model parameters at iteration~$t$        \\
$l(\bz; \btheta_t)$               & Loss for example $\bz$ at iteration $t$  \\
$\bg_t(\bz)$                      & Gradient $\nabla l(\bz; \btheta_t)$      \\
$\tilde{\mclb}_t \subset \dtra$   & Candidate batch                          \\
$\mclb_t \subset \tilde{\mclb}_t$ & Refined training batch                   \\
$q(\bz)$                          & Observed quality score                   \\
 $\be(\bz) = \be(\bx)$    & Embedding of example $\bz$  \\
$s(\bz)$                          & Latent quality (GP sample)               \\
$\mcls$                           & Set of $m$ strategies                    \\
$T \in \mathbb{N}$                 & Total number of training iterations      \\
$\kappa$                          & GP kernel                                \\
$\lambda \in\mathbb{R}_+^*$                         & Hedge learning rate                      \\
$\alpha \in [0,1) $                        & Mixing rate (fixed-share)                \\
$\tau \in\mathbb{R}_+^*$                            & Sampling temperature                     \\
$\beta \in\mathbb{R}_+^*$                           & Likelihood temperature                   \\
$k \in\mathbb{R}_+^*$                               & Number of top positions in likelihood    \\
$\eta_t \in\mathbb{R}_+^*$                          & SGD learning rate                        \\
\bottomrule
\end{tabular}
\end{table}

\subsection{Gaussian Process Background}\label{app:gp}
\label{subsec:gp}

We model the latent quality function using a \emph{Gaussian process}~(GP)~\cite{williams2006gaussian}.
Let~$s: \dtra \to \mathbb{R}$ denote the GP prior over quality scores, where~$s(\bz)$ represents the intrinsic, long-term quality of example~$\bz$. As a simple modeling choice, we assume the observed quality score~$q(\bz)$ is a noisy realization of~$s(\bz)$:
\begin{align}
    q(\bz) = s(\bz) + \epsilon\ , \quad \epsilon \sim \mathcal{N}(0, \sigma_n^2)\ ,
    \label{eq:gausslik}
\end{align}
where~$\sigma_n^2$ denotes the observation noise variance. The noise term captures iteration-specific fluctuations: since~$q(\bz)$ depends on the current model parameters~$\btheta_t$, which evolve during training, the same example may yield different quality scores at different iterations. We revisit this likelihood model in Section~\ref{sec:approach} to better capture the relationship between observed scores and the latent quality function.

Under the GP prior,~$\{s(\bz)\}_{\bz \in \dtra}$ follows a multivariate Gaussian distribution fully specified by a prior mean and a covariance kernel~$\kappa_{\bz, \bz'} \triangleq \text{cov}(q(\bz), q(\bz'))$.
The kernel encodes the assumption that similar training examples have similar quality scores.
To evaluate $\kappa_{\bz, \bz'}$, we embed the input component $\bx$ of each example $\bz = (\bx, y)$ using a pre-trained model and apply random projection to obtain a low-dimensional representation $\be(\bz) \triangleq \be(\bx)$. Note that the embedding depends only on the input, not the label.
Following standard practice in GP modeling, we employ the \emph{squared exponential} (SE) kernel, though other kernels may also be used:
\begin{align}
    \kappa_{\bz, \bz'} \triangleq \sigma_s^2 \exp \left( - \frac{1}{2} \Big\|\frac{\be(\bz) - \be(\bz')}{\bl^2} \Big\|^2 \right)\ ,
\end{align}
where~$\sigma_s$ and~$\bl$ denote the signal variance and lengthscales, respectively. These hyperparameters can be estimated directly from the data by maximizing the GP log marginal likelihood \cite{williams2006gaussian,artemev2021tighter}.
Let~$\mcld_t \triangleq \{(\bz, q(\bz))\}_{\bz \in \bigcup_{t'=1}^{t-1} \mclb_{t'}}$ denote the set of observed quality scores from previous iterations. Given~$\mcld_t$, the posterior distribution of~$s(\bz_*)$ for any~$\bz_* \in \dtra$ is Gaussian with closed-form mean~$\mu_t(\bz_*)$ and variance~$\sigma_t^2(\bz_*)$~\cite{williams2006gaussian}.
However, exact GP inference requires~$\mathcal{O}(n^3)$ time, where~$n$ is the number of observations, rendering it impractical for per-iteration evaluation. Even sparse GP approximations with~$n_{\text{ind}}$ inducing points require~$\mathcal{O}(n n_{\text{ind}}^2)$ time. Our approach (Section~\ref{sec:approach}) addresses this computational bottleneck while preserving the expressiveness of the GP framework for online data selection.

\subsection{Dynamic-Regret Guarantee via Fixed-Share Hedge}\label{app:regret}

This appendix supports Theorem~\ref{thm:regret-main} of the main text. We state the posterior update of Section~\ref{sec:approach} as an instance of the \emph{fixed-share Hedge} algorithm of \citet{herbster1998tracking} and instantiate the standard fixed-share Hedge regret bound for our setting. We condition on the set of strategies $\mcls$ constructed at initialization by paired sampling from $\text{GP}(0, \kappa)$ (see Section~\ref{sec:approach}) and analyze the regret of the posterior update over these fixed strategies, which is the component of the algorithm that is updated online. The GP sampling step is a one-time offline operation whose analysis is orthogonal to the regret argument.

\paragraph{Setup.} Let $\mcls$ be the finite strategy set of size $m$ constructed in Section~\ref{sec:approach} by paired sampling from $\text{GP}(0, \kappa)$, and let $T$ denote the total number of training iterations. At iteration $t \in \{1, \ldots, T\}$, the algorithm holds a distribution $p_t$ over $\mcls$ (defined as $p_t(s) \triangleq p(s \mid \mcld_{t-1})$, the posterior from Section~\ref{sec:approach}), samples a strategy, uses it to construct the candidate batch, and observes the per-strategy loss
\begin{align}
    \ell_t(s) \triangleq -\log p_{\text{top-}k}\bigl(q(\mclb_t) \mid s\bigr)\ .
    \label{eq:loss}
\end{align}
Note that $\ell_t(s)$ can be computed for every $s \in \mcls$ from the pre-computed values $\{s(\bz)\}_{\bz \in \mclb_t}$, so the algorithm operates in the \emph{full-information} setting rather than the bandit setting. Because $\mclb_t$ is itself generated using the sampled strategy $s_t \sim p_t$, the loss function $\ell_t$ is a random object determined by the algorithm's own trajectory; Theorem~\ref{thm:regret} therefore bounds regret along this realized trajectory rather than against a counterfactual trajectory driven by the comparator.

\paragraph{Loss boundedness.} Let $b \triangleq |\mclb_t|$. For a fixed set of strategies $\mcls$ sampled at initialization, define
\begin{align}
    S \triangleq \sup_{s \in \mcls,\, \bz \in \dtra} |s(\bz)|\ ,
\end{align}
which is a finite deterministic quantity that can be computed once after the strategies are drawn. We bound $\ell_t(s)$ by direct calculation. Let $a_i \triangleq s(\bz_{(i)}) / \beta$. Then
\begin{align}
    \ell_t(s) = \sum_{j=1}^k \log \sum_{i=1}^b e^{a_i - a_j}\ .
\end{align}
The lower bound $\ell_t(s) \geq 0$ follows because the inner sum includes the term $i = j$, which equals $1$, so each log term is non-negative. The upper bound uses $a_i - a_j \in [-2 S / \beta, 2 S / \beta]$, giving $\sum_i e^{a_i - a_j} \leq b \cdot e^{2 S / \beta}$ and therefore $\ell_t(s) \leq L$ with
\begin{align}
    L \triangleq k \bigl( \log b + 2 S / \beta \bigr)\ .
    \label{eq:L}
\end{align}

\paragraph{Connection to fixed-share Hedge.} We show that the posterior update of Section~\ref{sec:approach} is a particular instance of fixed-share Hedge.

\begin{proposition}[Equivalence to fixed-share Hedge]
\label{prop:equiv}
The two-step posterior update of Section~\ref{sec:approach},
\begin{align}
    \hat{p}(s | \mcld_t) &\propto p(s | \mcld_{t-1})\, p_{\text{top-}k}\bigl(q(\mclb_t) \mid s\bigr)^\lambda\ ,\\
    p(s | \mcld_t) &= \alpha / m + (1 - \alpha)\, \hat{p}(s | \mcld_t)\ ,
\end{align}
where $\hat{p}(s | \mcld_t)$ denotes the intermediate post-Hedge distribution and $p(s | \mcld_t)$ is the final posterior carried into the next iteration, is equivalent to the fixed-share Hedge algorithm of \citet{herbster1998tracking} with learning rate $\lambda$ and mixing rate $\alpha$, applied to the loss sequence $\ell_t(s) = -\log p_{\text{top-}k}(q(\mclb_t) \mid s)$.
\end{proposition}

\begin{proof}
Substituting $p_{\text{top-}k}(q(\mclb_t) \mid s) = \exp(-\ell_t(s))$ into the first step gives $\hat{p}(s | \mcld_t) \propto p(s | \mcld_{t-1})\, \exp(-\lambda \ell_t(s))$, which is the exponential-weights (Hedge) update with learning rate $\lambda$ on the loss $\ell_t$. The second step is the uniform mixing step of fixed-share Hedge with mixing rate $\alpha$ and uniform prior $1/m$.
\end{proof}

\paragraph{Dynamic-regret comparator.} To capture non-stationarity of the optimal strategy, we compare the algorithm's cumulative loss to that of an arbitrary comparator sequence $\bs^* = (s_1^*, \ldots, s_T^*) \in \mcls^T$ that is allowed to switch up to $K$ times, i.e.\ $|\{t : s_{t+1}^* \neq s_t^*\}| \leq K$. The dynamic regret against $\bs^*$ is
\begin{align}
    \mathrm{Reg}_T(\bs^*) \triangleq \sum_{t=1}^T \mathbb{E}_{s \sim p_t}[\ell_t(s)] - \sum_{t=1}^T \ell_t(s_t^*)\ .
\end{align}

\paragraph{Statement of the bound.}

\begin{theorem}[Dynamic regret bound, instantiation of \citealp{herbster1998tracking} for \gaia{}]
\label{thm:regret}
Fix the set of strategies $\mcls$ sampled at initialization, which determines $S$ and hence $L$. Let $\bs^* \in \mcls^T$ be any comparator sequence with at most $K$ switches. Running the posterior update of Section~\ref{sec:approach} with learning rate
\begin{align}
    \lambda^* = \sqrt{\frac{8 \bigl( K \log m + K \log(T / K) + \log m \bigr)}{L^2 T}}
\end{align}
and mixing rate $\alpha^* = K / (T - 1)$ yields
\begin{align}
    \mathrm{Reg}_T(\bs^*) \leq L \sqrt{\frac{T}{2} \bigl( K \log m + K \log(T / K) + \log m \bigr)}\ ,
    \label{eq:regret}
\end{align}
where $L = k (\log b + 2 S / \beta)$ is the loss bound from Equation~\eqref{eq:L}.
\end{theorem}

\begin{proof}
By Proposition~\ref{prop:equiv}, the posterior update of Section~\ref{sec:approach} is fixed-share Hedge applied to the loss sequence $(\ell_t)_{t=1}^T$. The losses satisfy $\ell_t(s) \in [0, L]$ per Equation~\eqref{eq:L}, and the algorithm operates in the full-information setting because $\ell_t(s)$ is computable for every $s \in \mcls$ from $\{s(\bz)\}_{\bz \in \mclb_t}$. The bound then follows from the standard dynamic-regret analysis for fixed-share Hedge with bounded losses (see \citet{herbster1998tracking} for the algorithm and \citet{cesa2006prediction} for the $[0,L]$-loss analysis via Hoeffding-type exponential-weights bounds), instantiated with the stated choices of $\lambda^*$ and $\alpha^*$.
\end{proof}

\subsection{Training Tasks}\label{app:tasks}

We follow the experimental setting in \greats, reported our method on three instruction finetuning tasks:

\begin{itemize}
    \item \textbf{\mmlu} Finetuning \llama and \qwen on 5\% \less data, and evaluate on  \mmlu, we target the query, key, value, and output projections, setting the LoRA rank to 128, LoRA $\alpha$ to 1.0, and dropout to 0.1. The learning rate is set to 2e-5.
    \item \textbf{\samsum} Finetuning \llamat and \qwen on \alpaca dataset, and evaluate on \samsum, we target the the query, value, and output projections within the self-attention layers, as well as the gate, up, and down projections of the feed-forward networks, setting the LoRA rank to 8, LoRA $\alpha$ to 16.0, and dropout to 0.1. The learning rate is set to 2e-5.
    \item \textbf{\tydiqa} Finetuning \mistral on 40\% \less data, and evaluate on \tydiqa, we target the query and key attention projections, setting the LoRA rank to 128, LoRA $\alpha$ to 1.0, and dropout to 0.1. The learning rate is set to 1e-5.
\end{itemize}

\textbf{Training Dataset}
There are two primary instruction tuning datasets are involved in our experiments: \less~\cite{xia2024less} and \alpaca~\cite{alpaca}.
\begin{itemize}
    \item \textbf{\less:} A composite dataset combining four instruction tuning sources: \flan~\cite{flan}, \datacot~\cite{cot}, \dolly~\cite{dolly}, and \openassistant~\cite{kopf2023openassistant}.
    \item \textbf{\alpaca:} An instruction-following dataset comprising 52k examples, covering a diverse range of user-oriented tasks.
\end{itemize}

\textbf{Evaluation Dataset}
To ensure a comprehensive evaluation, we assess performance on three distinct benchmarks: \mmlu, \tydiqa, and \samsum. We report perplexity for all tasks

\begin{itemize}
    \item \textbf{\mmlu} consists of multiple-choice questions across 57 subjects. While our main experiments focus on the Sociology subject, we report accuracy across 5 different subjects to demonstrate robustness.
    \item \textbf{\tydiqa} is a multilingual question-answering benchmark covering 11 diverse languages, where the objective is to extract an answer from a provided passage. For evaluation, we randomly select a held-out set of 500 test samples.
    \item \textbf{\samsum} involves summarizing dialogues. For this task, we report Perplexity, and ROUGE scores to measure generation quality.
\end{itemize}

\subsection{Implementation Details}\label{app:implementation}

\paragraph{GP Training and Hyperparameters}
Our Gaussian Process model is implemented using the \texttt{GPyTorch} library \cite{gardner2018gpytorch}. We optimize the kernel hyperparameters---specifically the signal variance $\sigma_s$ and lengthscales $\bl$---by maximizing the marginal log-likelihood (MLL) of the warm-up observations.

The GP training involves $200$ iterations with an initial noise standard deviation of $0.1$ and a Hedge learning rate of $\lambda = 0.8$. For the adaptive temperature schedule, we update the temperature every $100$ training steps with an increment of $0.01$ per interval. The semantic embeddings are projected into a $10$-dimensional subspace using SVD to ensure efficient posterior sampling.

\paragraph{Task-Specific Configurations}
We provide the detailed strategy pool sizes and temperature settings for each benchmark in Table~\ref{tab:hyperparams}.

\begin{table}[h]
\centering
\caption{Hyperparameter settings for \gaia{} across different tasks.}
\label{tab:hyperparams}
\begin{tabular}{lccc}
\toprule
\textbf{Task} & \textbf{Strategy Pool Size ($m$)} & \textbf{Sampling Temp ($\tau$)} & \textbf{Likelihood Temp ($\beta$)} \\
\midrule
MMLU    & 20  & 0.2 & 0.1  \\
Samsum  & 100 & 0.1 & 0.01 \\
TyDiQA  & 150 & 0.1 & 0.01 \\
\bottomrule
\end{tabular}
\end{table}

\subsection{Baselines and Scoring Functions}\label{app:baselines}
We compare \gaia{} against several representative data selection and valuation methods:

\begin{itemize}[nosep]
    \item \textbf{\greats:}~\cite{wang2024greats} An efficient online batch selection method that employs a greedy algorithm to optimize data batch quality, approximated via Taylor expansion.
    \item \textbf{\gradnorm:}~\cite{gradnorm} A selection strategy that prioritizes training data points exhibiting the highest gradient norms.
    \item \textbf{\maxloss:}~\cite{maxloss} A strategy that selects training data points with the highest loss values, focusing on "hard" examples.
    \item \textbf{\sbert:}~\cite{wang2024greats} A semantic selection method that curates training batches based on their embedding similarity to the validation set, utilizing Sentence-BERT~\cite{reimers-2019-sentence-bert}.
    \item \textbf{\rholoss:}~\cite{rholoss} A selection strategy that selects data points by maximizing the difference between the target model's loss and a reference model's loss, thereby filtering out noise while prioritizing learnable examples.
\end{itemize}

\subsection{The Other Surrogate Models: Why Gaussian Processes?}

To demonstrate the superiority of Gaussian Processes for guiding data sampling, Figure~\ref{fig:surrogate} evaluates the downstream performance of alternative sampling pipelines driven by standard parametric models (Linear Regression, MLP, and Transformer). We trained these alternative surrogates on the 500 warm-up samples to predict global utility scores, followed by data selection using either a deterministic curriculum (Curric) or probabilistic (Prob) sampling strategy. The results clearly show that these alternative sampling methods fall short. \gaia{} significantly and consistently outperforms all baseline sampling configurations across all downstream metrics, achieving substantially lower validation and test perplexity, as well as higher test accuracy. This confirms that our GP-based framework provides a much more robust and effective utility signal for data selection, translating directly to superior downstream task performance
\begin{figure*}
    \centering
    \includegraphics[width=\linewidth]{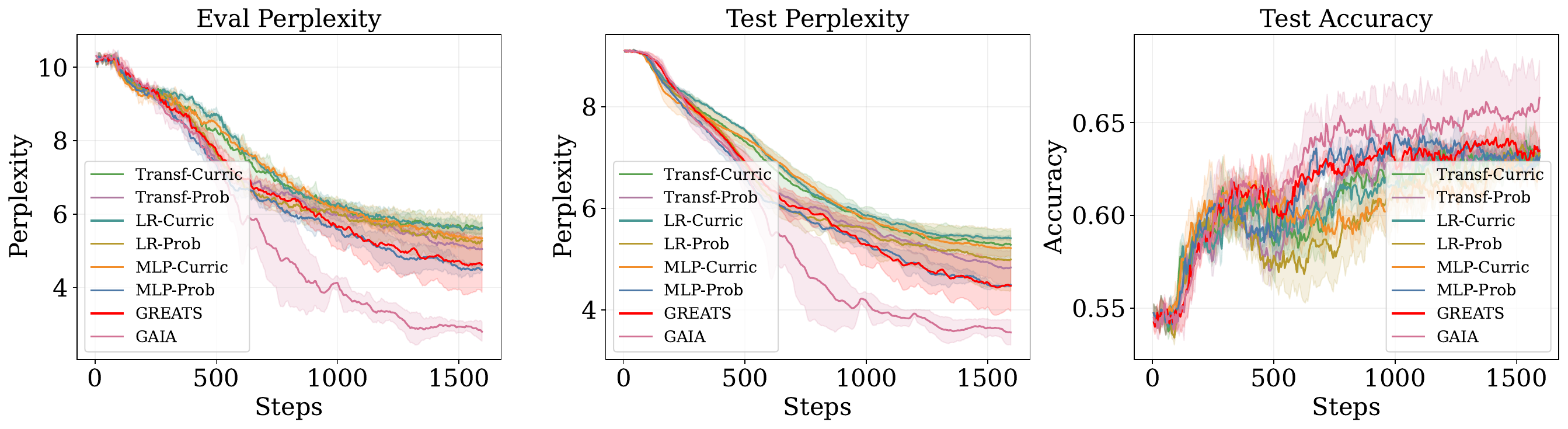}
    \caption{Training dynamics on \mmlu-Sociology using the \llama backbone, comparing downstream performance when using alternative surrogate models for data sampling.}

    \label{fig:surrogate}
\end{figure*}

\subsection{\mmlu-Others}
As shown in Table \ref{tab:results_ppl}, \gaia achieves the lowest average perplexity (5.9201), significantly outperforming the baseline (7.8854) and \greats{} (6.6981). These results confirm that \gaia’s data selection strategy more effectively captures specialized knowledge than existing approaches.

\begin{table}[t]
\centering
\caption{Comparison of Test Perplexity on MMLU subsets. \textbf{AVG} denotes the mean performance across all five subjects. Bold values indicate the best performance (lower is better).}
\label{tab:results_ppl}
\small
\begin{tabular}{lcccccc}
\toprule
\textbf{Method} & \textbf{Anat.} & \textbf{Astro.} & \textbf{B. Ethics} & \textbf{Clin. Know.} & \textbf{Abs. Alg.} & \textbf{AVG} \\ 
\midrule
Regular & 7.4797 & 7.7748 & 8.0511 & 7.5611 & 8.5603 & 7.8854 \\
\greats{}  & 7.2278 & \textbf{6.1785} & 6.5014 & 6.1407 & 7.4420 & 6.6981 \\
\gaia{}    & \textbf{5.8763} & 6.2896 & \textbf{5.5351} & \textbf{5.2791} & \textbf{6.6205} & \textbf{5.9201} \\ 
\bottomrule
\end{tabular}
\end{table}

\subsection{Performance on \tydiqa{}}\label{app:tydiqa}
\begin{figure*}
    \centering
    \includegraphics[width=\linewidth]{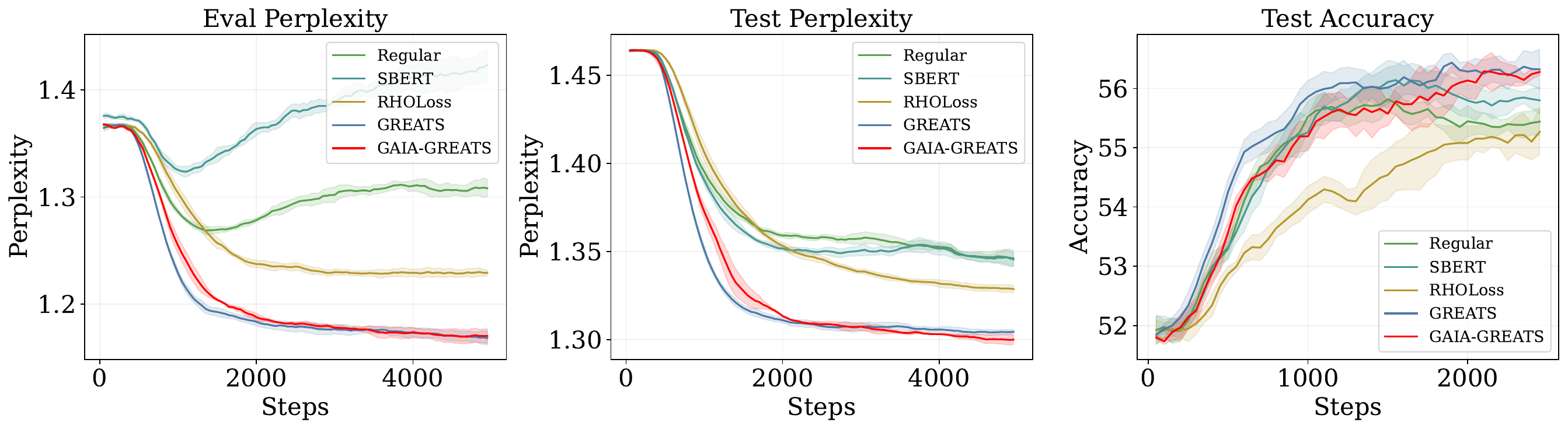}
    \caption{Training dynamics on the \tydiqa{} benchmark. Our method achieves better perplexity on the test set. All results are averaged over three independent runs.}

    \label{fig:tydiqa}
\end{figure*}

Figure~\ref{fig:tydiqa} illustrates the convergence and downstream performance on the \tydiqa dataset. Although \textbf{GP-\greats{}} exhibits a slightly more gradual decline in evaluation and test perplexity during the initial training phase, it eventually achieves a superior convergence floor compared to \greats{}. Specifically, the final perplexity levels (both eval and test) of our method are marginally lower than those of the strongest baseline. Furthermore, in terms of downstream task performance, our method maintains an F1 score that is on par with \greats{}. This suggests that while our global selection strategy may take longer to identify optimal data patterns on certain datasets, it ultimately reaches a more refined state of model alignment, demonstrating competitive generalization capabilities across diverse linguistic tasks.

\begin{figure*}[t]
    \centering
    \begin{subfigure}{\linewidth}
        \centering
        \includegraphics[width=0.95\linewidth]{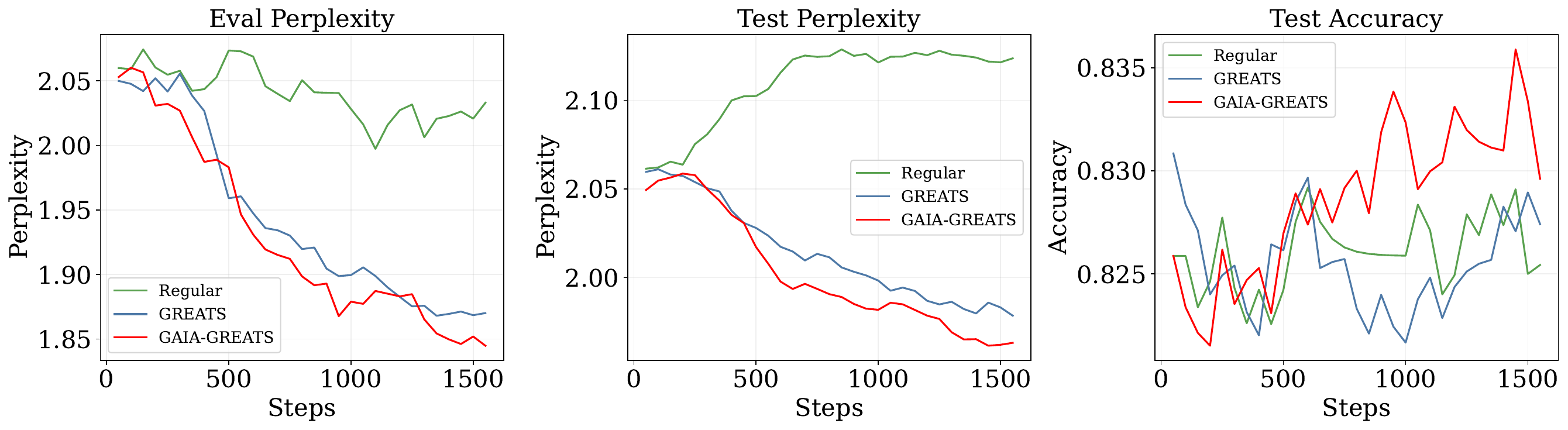}
        \caption{MMLU - Sociology: Validation Perplexity, Test Perplexity, and Accuracy (from left to right)}
        \label{fig:qwen-soc}
    \end{subfigure}

    \vspace{0.2em}

    \begin{subfigure}{\linewidth}
        \centering
        \includegraphics[width=0.95\linewidth]{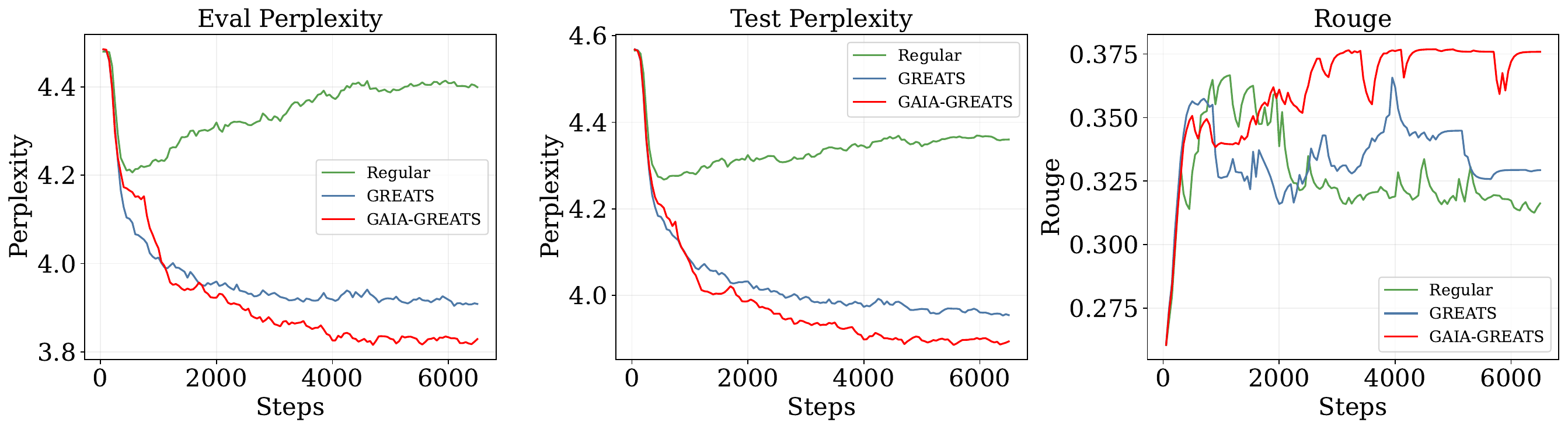}
        \caption{SAMSum: Validation Perplexity, Test Perplexity, and Rouge (from left to right)}
        \label{fig:qwen-sam}
    \end{subfigure}

    \caption{Training dynamics on \mmlu-Sociology and \samsum using the Qwen3-4B backbone. We compare \gaia{}-\greats against Regular training and the GREATS baseline to evaluate cross-architecture generalization.}
    \label{fig:qwen}
    \vspace{-15pt}
\end{figure*}
\subsection{Qwen model}\label{app:qwen}

To verify that our framework is architecture-agnostic, we extend our evaluation to the Qwen model family. As illustrated in Figure~\ref{fig:qwen}, \gaia{}-\greats{} consistently maintains its performance superiority over both Regular training and the strong \greats{} baseline. On the \mmlu-Sociology task, \gaia{} achieves a noticeably faster reduction in both evaluation and test perplexity, which directly translates to a higher and more stable downstream accuracy. Similarly, on the \samsum summarization task, \gaia{} exhibits a sharper initial descent in perplexity and secures a higher ROUGE score during the early training stages. These results confirm that the global utility manifold modeled by our Gaussian Process effectively captures intrinsic data quality independently of the underlying LLM architecture, ensuring robust generalization across different foundation models.

\subsection{Performance on Additional \mmlu Subjects}\label{app:mmlu_addition}

\begin{figure*}[t]
    \centering
    \begin{subfigure}{0.48\textwidth}
        \centering
        \includegraphics[width=\linewidth]{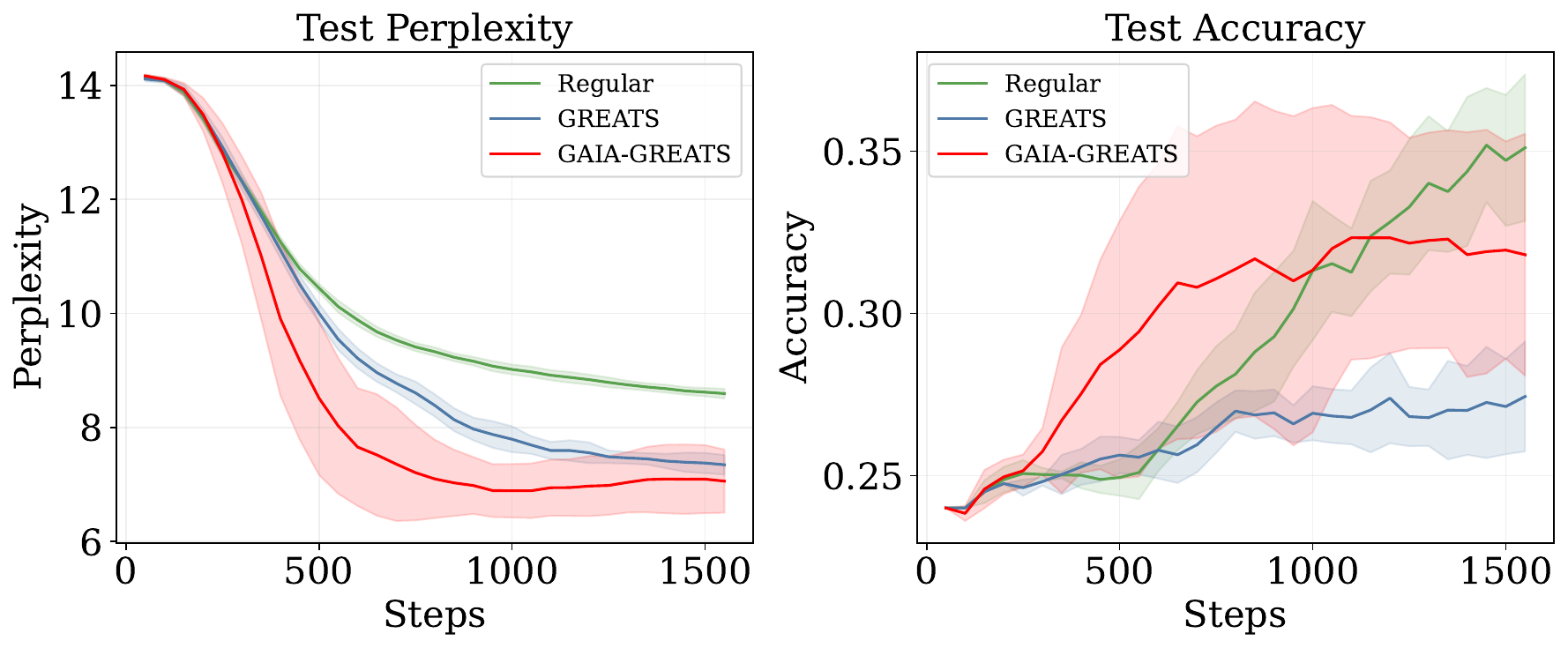}
        \caption{Abstract Algebra}
        \label{fig:res_algebra}
    \end{subfigure}
    \hfill
    \begin{subfigure}{0.48\textwidth}
        \centering
        \includegraphics[width=\linewidth]{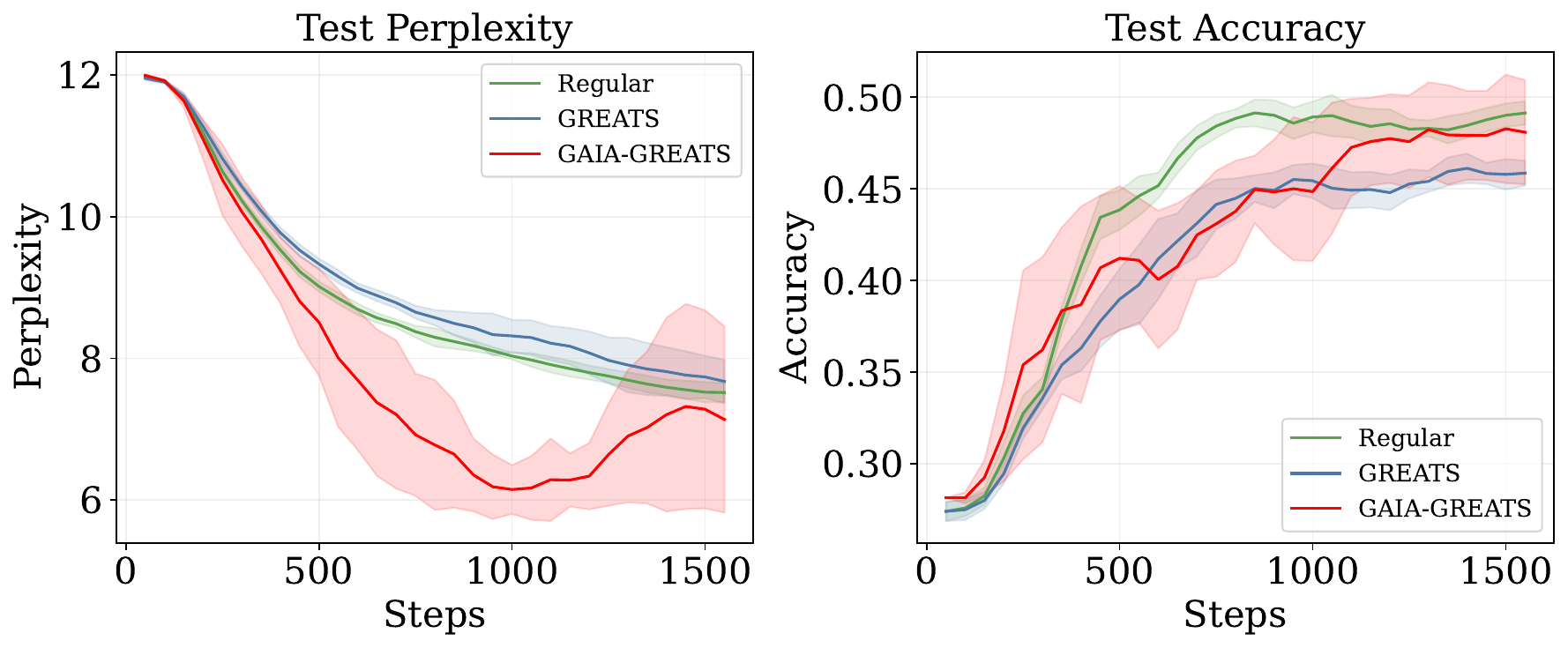}
        \caption{Anatomy}
        \label{fig:res_anatomy}
    \end{subfigure}

    \vspace{10pt} 

    \begin{subfigure}{0.48\textwidth}
        \centering
        \includegraphics[width=\linewidth]{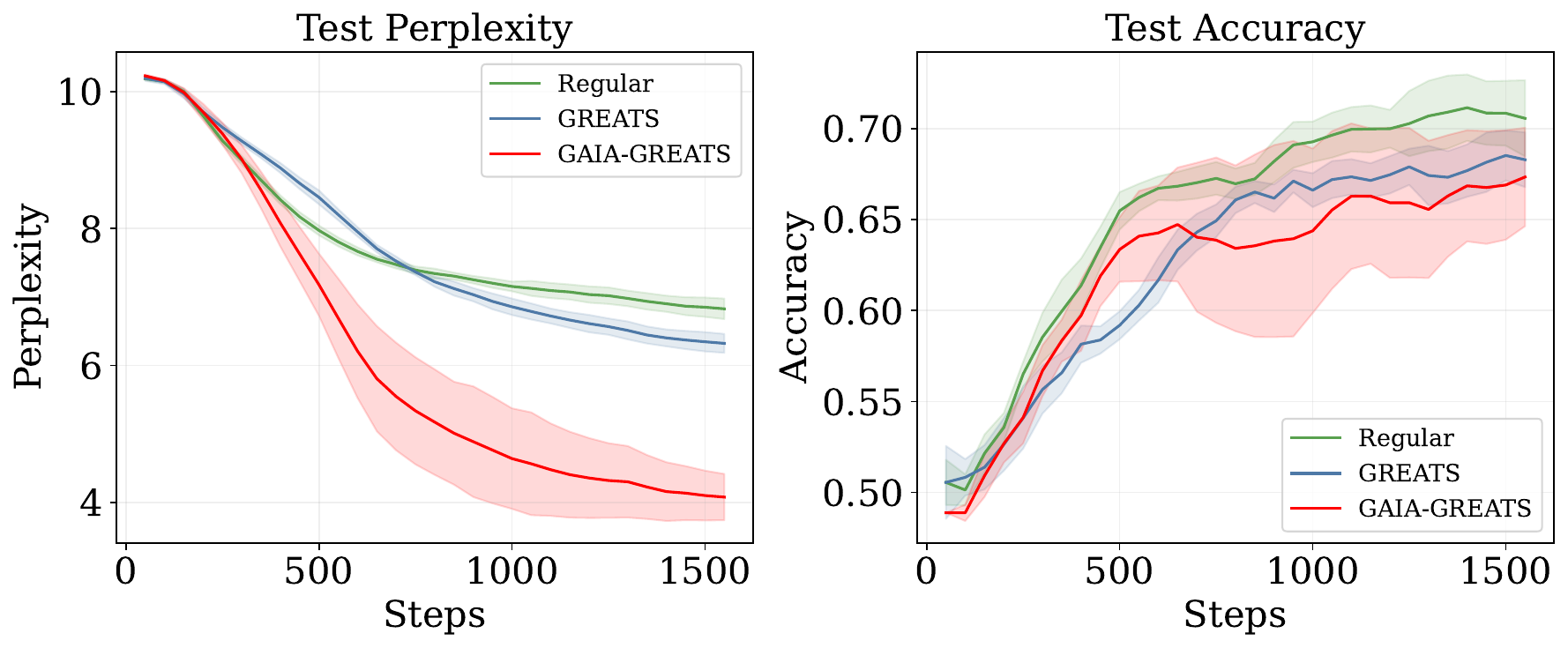}
        \caption{US Foreign Policy} 
        \label{fig:res_policy}
    \end{subfigure}
    \hfill
    \begin{subfigure}{0.48\textwidth}
        \centering
        \includegraphics[width=\linewidth]{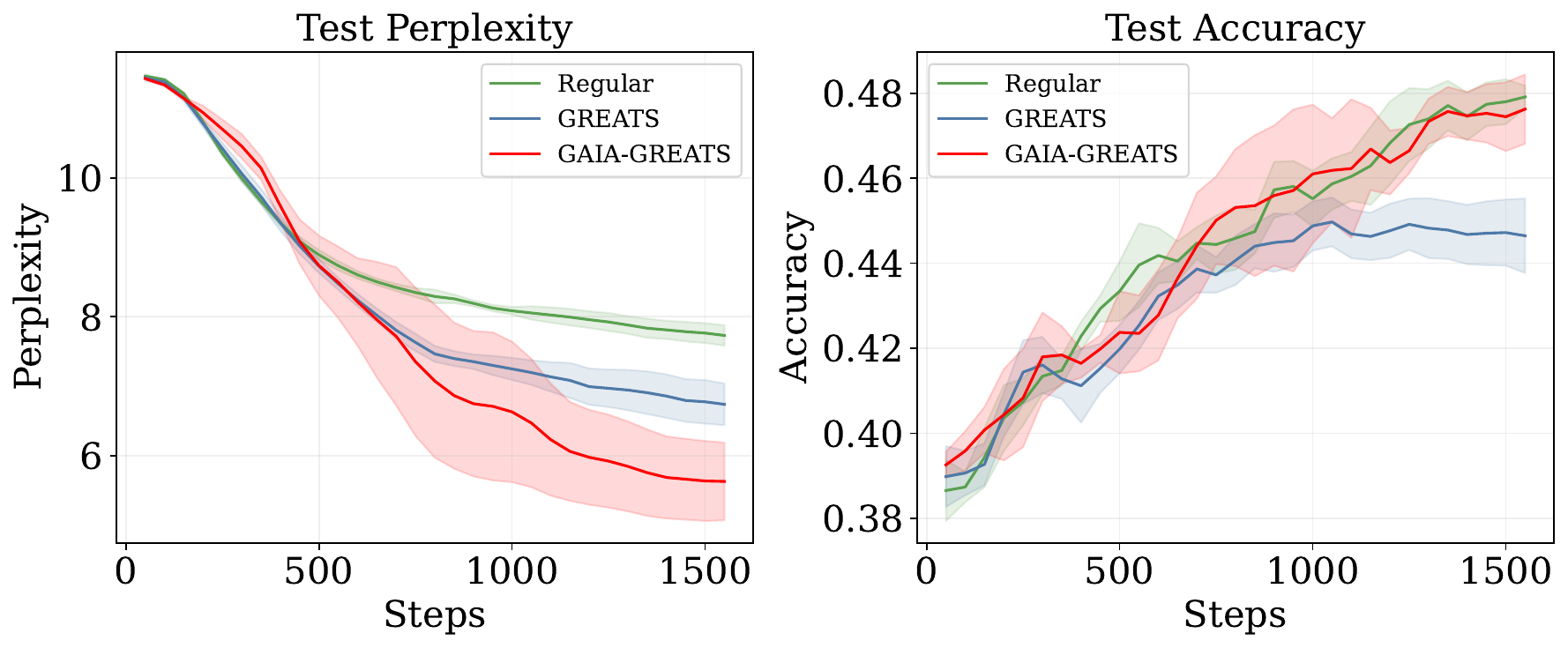}
        \caption{Astronomy}
        \label{fig:res_astronomy}
    \end{subfigure}

    \vspace{-5pt}
    \caption{Performance comparison on additional \mmlu{} subjects (\textit{Abstract Algebra, Anatomy, US Foreign Policy, Astronomy}). We evaluate \gaia{}-\greats{} against \textit{Regular} training and the \greats{} baseline, reporting both test perplexity (PPL) and accuracy. Notably, our method achieves a significant and consistent reduction in PPL across all tasks. While the accuracy metrics exhibit some variance, typical for the inherent stochasticity of multiple-choice tasks, \gaia{}-\greats{} maintains competitive performance.}
    \label{fig:additional_mmlu}
    \vspace{-10pt}
\end{figure*}

Figure~\ref{fig:additional_mmlu} presents a detailed performance comparison across various \textsc{MMLU} subjects, including \textit{Abstract Algebra, Anatomy, US Foreign Policy, and Astronomy}.
Our method \textbf{substantially outpaces} both the Regular training and \greats{} baselines, establishing significantly faster convergence and achieving a \textbf{noticeably lower convergence floor} in terms of test perplexity (PPL).
This decisive advantage in PPL underscores our method's superior efficiency in identifying and learning from high-utility data points.
Regarding downstream accuracy, we observe a higher degree of variance compared to the baselines.
This phenomenon stems from our framework's inherent mechanism: by utilizing \greats{} as the scoring function---which is fundamentally designed for loss maximization---our approach \textbf{amplifies the optimization of PPL}.
While this leads to a superior state of language modeling, it simultaneously introduces instability in accuracy, a metric further sensitized by the discrete nature of multiple-choice tasks.
Nevertheless, the overall results demonstrate that our framework possesses \textbf{formidable convergence capabilities} and robust generalization, maintaining highly competitive performance across all benchmarks.

\subsection{Computational Efficiency.} \label{app:time}

\begin{table}[t]
    \centering
    \footnotesize
    \caption{Analysis of the additional computational overhead introduced by our method across different benchmarks. We report the training set size  and the specific incremental costs: embedding acquisition and Gaussian Process training. All time measurements are reported in seconds.}
    \label{tab:time_efficiency}
    \begin{footnotesize}
    \begin{sc}
    \begin{tabular}{lccc}
        \toprule
        \textbf{Task} & \textbf{Size} & \textbf{Emb. Time (s)} & \textbf{GP Time (s)} \\
        \midrule
        \mmlu & 13,533 & 11.68 & 11.20 \\
        \samsum & 52,002 & 17.73 & 10.54 \\
        \tydiqa & 108,271 & 67.33 & 11.55 \\
        \bottomrule
    \end{tabular}
    \end{sc}
    \end{footnotesize}
\end{table}

The computational overhead of our framework is negligible compared to the overall training time. As shown in Table~\ref{tab:time_efficiency}, the additional latency introduced by our method consists solely of embedding acquisition and GP training. Even on our largest dataset (\tydiqa{}, $N$=108,271), these steps combined require less than 80 seconds, while on \mmlu{}, the total overhead is merely $\approx$23 seconds. Furthermore, given that subsequent sampling and probability update operations are of $O(1)$ complexity, these processes impose a marginal burden relative to the backpropagation costs of Large Language Models. Since the \greats{} baseline already maintains a training time comparable to standard full-data training \cite{wang2024greats}, and \gaia{}-\greats{} operates as a lightweight wrapper over these signals, our method achieves superior data selection without introducing significant latency.

\subsection{Visualization of Sampling Dynamics.}\label{app:sampling_dynamics}
Figure~\ref{fig:strategy_scatter} visualizes the indices of sampled data points and their corresponding source strategies to illustrate the sampling behavior of our framework. We observe that the sampling distribution evolves significantly throughout the training process. In the initial phase, the framework exhibits a strong concentration on specific data regions, most notably the dense clusters sampled from strategies 2 and 10, facilitating \textbf{rapid targeted learning} on high-utility samples. As training progresses, however, the sampling pattern naturally shifts towards a more dispersed distribution, covering a broader range of the dataset. This temporal evolution demonstrates that our framework effectively prioritizes efficient convergence through focused exploitation in the early stages, before seamlessly transitioning to broader exploration to ensure sufficient data diversity for robust model generalization.

\begin{figure*}
    \centering
    \includegraphics[width=\linewidth]{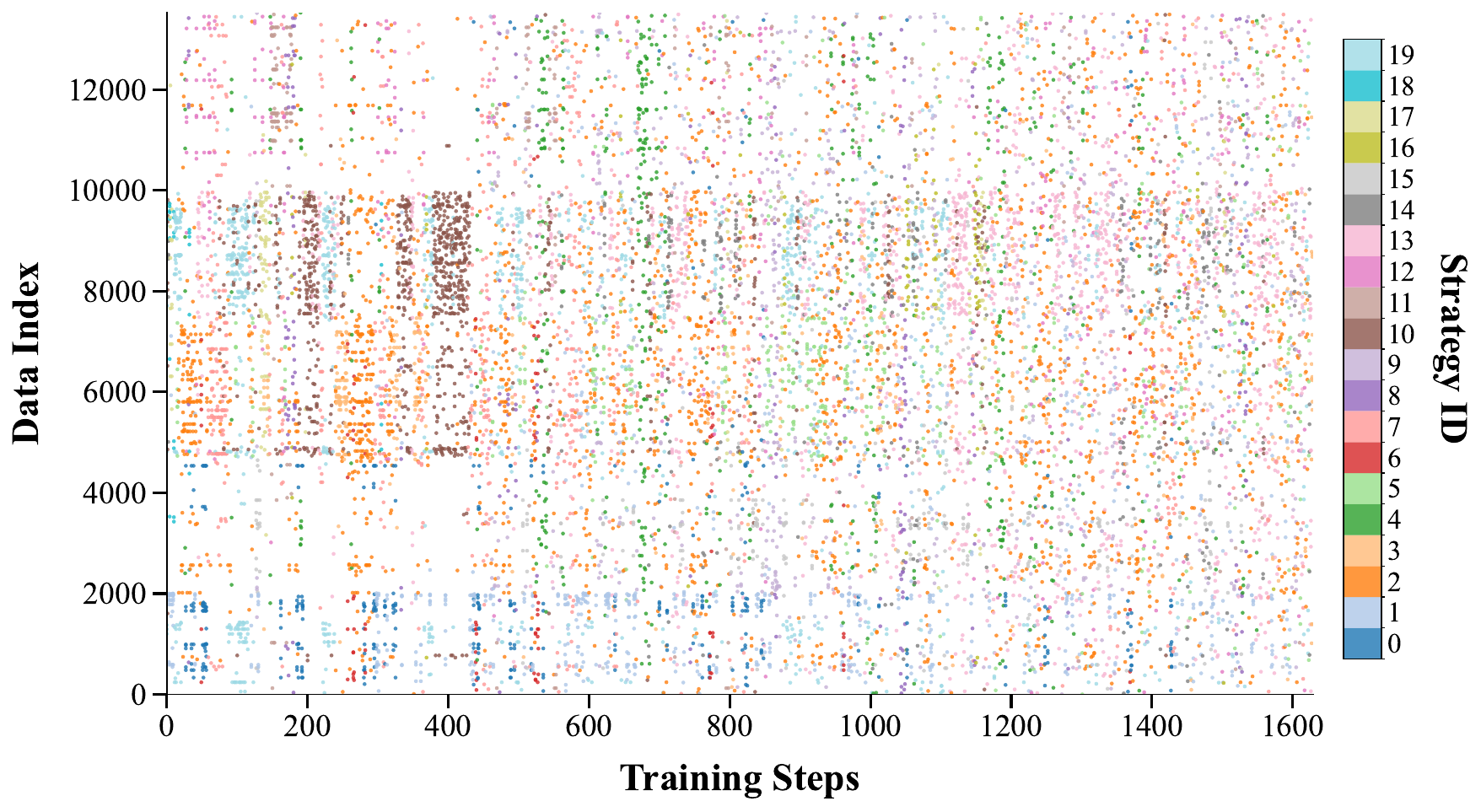}
    \caption{Visualization of sampling dynamics on \mmlu-Sociology over the course of training. The scatter plot maps sampled data indices (y-axis) against training steps (x-axis), with colors representing the distinct strategies from which samples were drawn. The distinct clustering observed in the initial phase demonstrates that our method prioritizes high-value data first. As training progresses, the selection becomes increasingly broad, thereby achieving greater diversity in the later stages.}
    \label{fig:strategy_scatter}
\end{figure*}

\begin{figure}[ht]
    \centering
    \begin{subfigure}{0.48\linewidth}
        \centering
        \includegraphics[width=\linewidth]{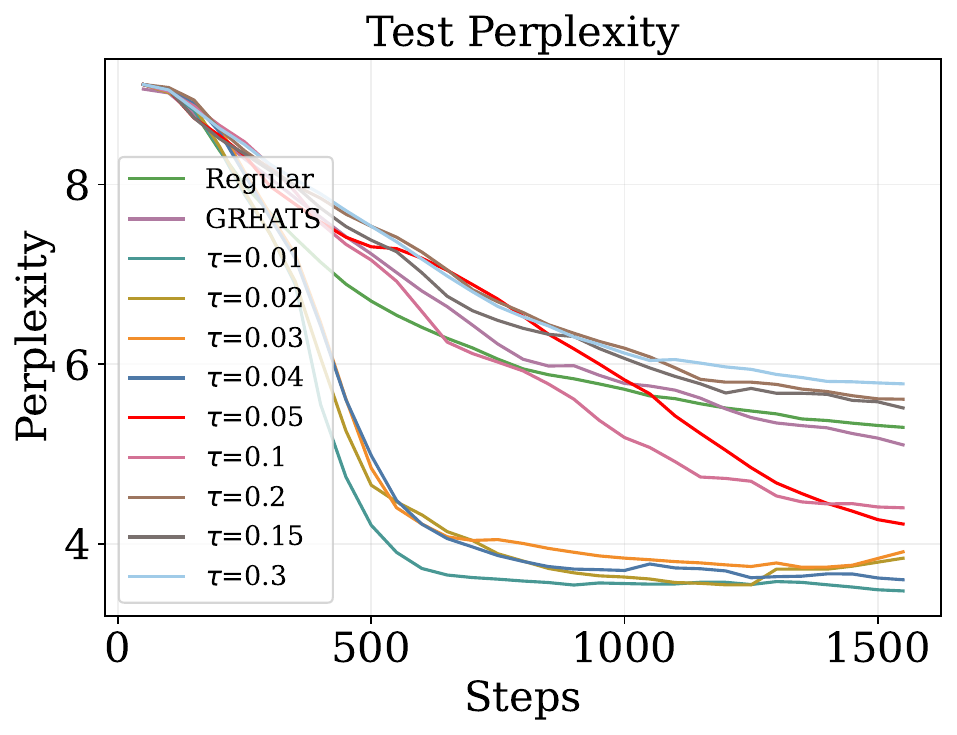}
        \caption{\mmlu-Sociology}
        \label{fig:t_sociology}
    \end{subfigure}\hfill 
    \begin{subfigure}{0.48\linewidth}
        \centering
        \includegraphics[width=\linewidth]{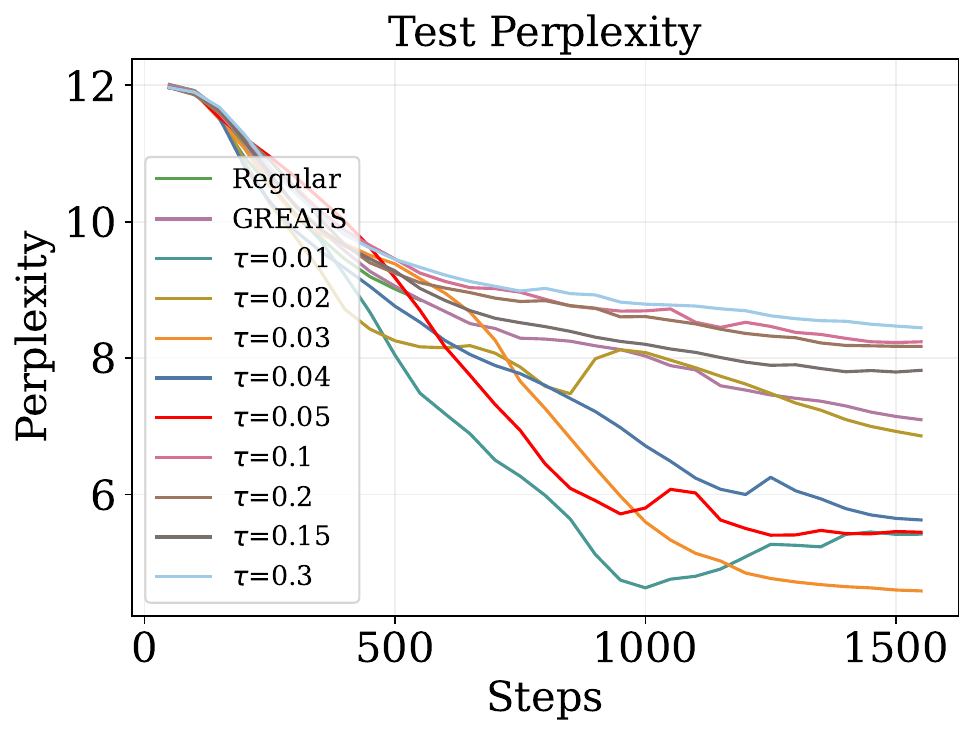}
        \caption{\mmlu-Anatomy}
        \label{fig:t_anatomy}
    \end{subfigure}
    \caption{Impact of sampling temperature on test perplexity for \mmlu-Sociology and \mmlu-Anatomy.}
    \label{fig:lt_impact}
\end{figure}

\subsection{Effectiveness of Adaptive Sampling Temperature.}\label{app:adaptive_sensitivity}
To balance exploration and exploitation during the sampling process, we introduce an \textbf{adaptive temperature} mechanism (See Section~\ref{sec:batch_sampling}). In the early stages of training, a lower temperature is employed to prioritize high-utility samples, facilitating rapid model alignment. As training progresses, the temperature is gradually increased to incorporate greater data diversity and prevent overfitting.

To verify its importance, we conduct an ablation study by disabling the adaptive schedule and using a fixed temperature instead. As shown in Figure~\ref{fig:scoring}, \textbf{w/o TA} results indicate the absence of adaptive temperature leads to significant training instability, with the model often experiencing performance \textbf{corruption} during the middle stages of training. This empirical evidence confirms that the adaptive temperature mechanism is vital for maintaining a stable optimization trajectory. It effectively prevents the model from over-focusing on a narrow subset of data early on, while ensuring sufficient diversity in later stages to mitigate overfitting and ensure robust convergence.

\subsection{Effect of sampling temperature.}\label{app:adaptive_mech}
The sampling temperature, $\tau$, plays a critical role in governing the sharpness of the sampling distribution derived from the Gaussian Process. As illustrated in Figure~\ref{fig:lt_impact}, results on \mmlu-Sociology and \mmlu-Anatomy, we observe a consistent pattern across different subjects. A lower temperature (e.g., $\tau=0.1$) leads to a significantly steeper and more rapid decline in test perplexity during the initial training phase. This is because a lower temperature amplifies the differences between data valuation scores, encouraging the model to aggressively exploit high-utility samples.

In contrast, a higher temperature results in a more gradual and flatter convergence curve. By smoothing the probability distribution, a larger $\tau$ increases the diversity of the selected batches but slows down the identification of the most informative data points. While lower temperatures accelerate early-stage convergence, they require a reliable valuation signal to avoid over-focusing on a narrow data subset. These results demonstrate that the sampling temperature serves as a vital trade-off parameter between aggressive exploitation of perceived data utility and the maintenance of batch diversity.




\subsection{Embedding methods}\label{app:embedding_dim}

We ablate the proxy semantic space using three different sentence transformers in Figure~\ref{fig:embs}. While all models successfully guide optimization, richer extractors like \texttt{all-mpnet-base-v2} and \texttt{multi-qa} yield a lower test perplexity floor than the lightweight \texttt{all-MiniLM-L6}. This confirms \gaia{}'s robustness to embedding choices, though higher-quality semantic representations explicitly improve the final optimization bounds.

\begin{figure*}[t]
    \centering
    \begin{subfigure}{0.48\textwidth}
        \centering
        \includegraphics[width=\linewidth]{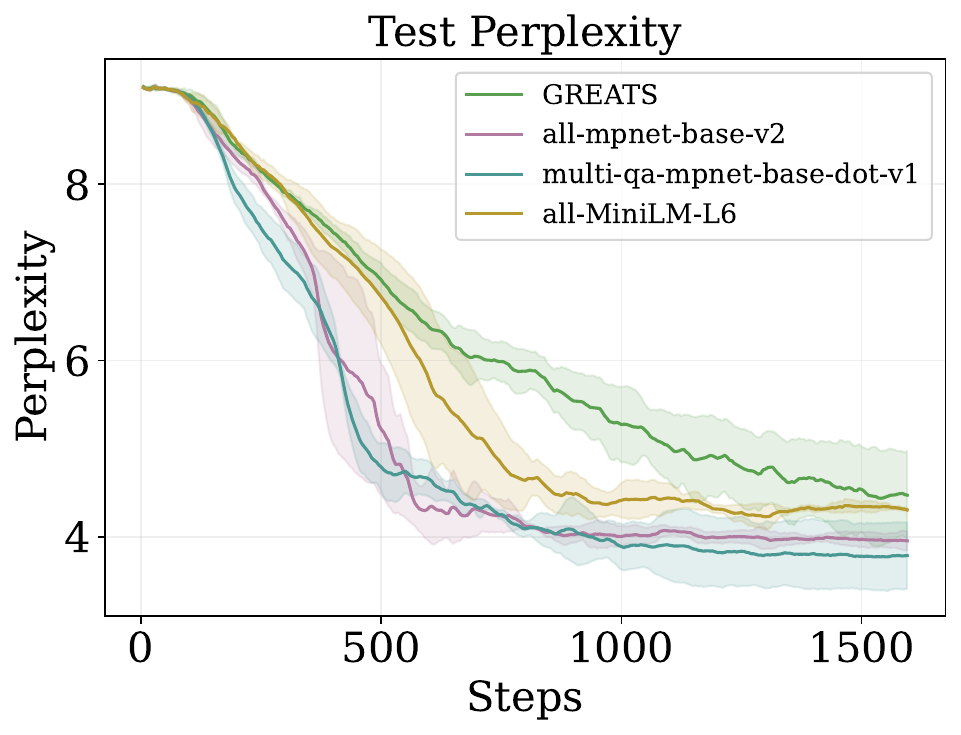}
        \caption{Effect of different pre-trained embedding models on test perplexity (\mmlu-Sociology)}
        \label{fig:embs}
    \end{subfigure}
    \hfill
    \begin{subfigure}{0.48\textwidth}
        \centering
        \includegraphics[width=\linewidth]{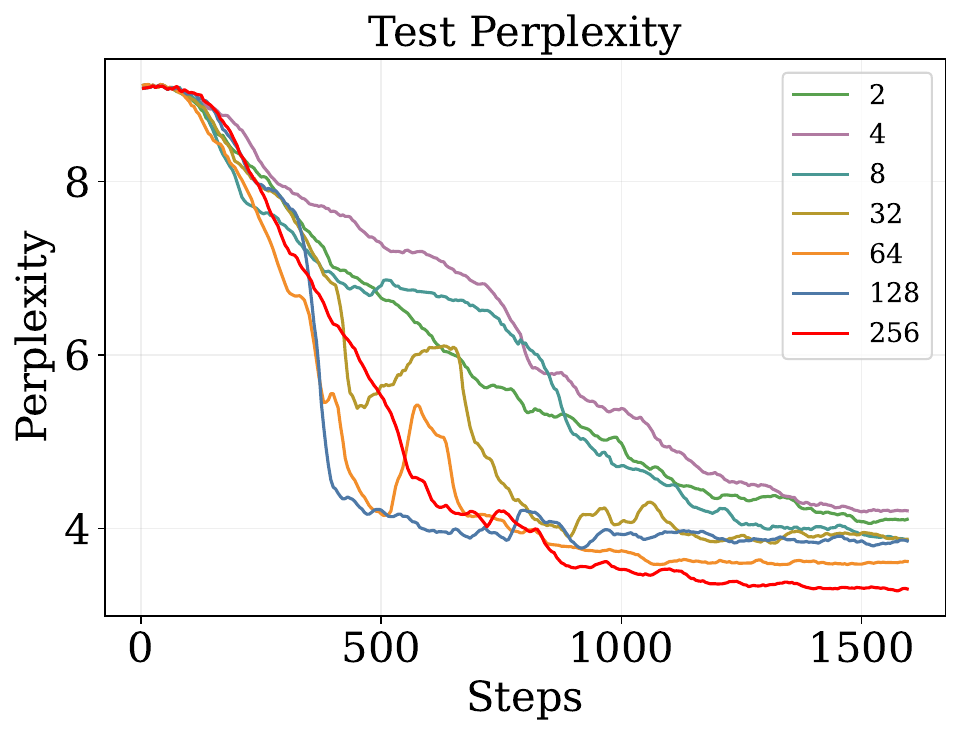}
        \caption{Effect of SVD projection dimensionality ($d$) on test perplexity (\mmlu-Sociology)}
        \label{fig:dim}
    \end{subfigure}
\end{figure*}



\subsection{Embedding Dimension}\label{app:embedding_methods}
To accelerate GP inference, \gaia{} uses SVD to project high-dimensional embeddings into a compact space. As shown in Figure~\ref{fig:dim}, we observe that the model's performance consistently improves as the projection dimension increases. Higher dimensions preserve more comprehensive semantic information from the raw embeddings, leading to faster convergence and lower test perplexity. 


\subsection{Warm-up: Impact of Initial Observation Size}\label{app:warmup}
\begin{figure}
    \centering
    \includegraphics[width=0.6\linewidth]{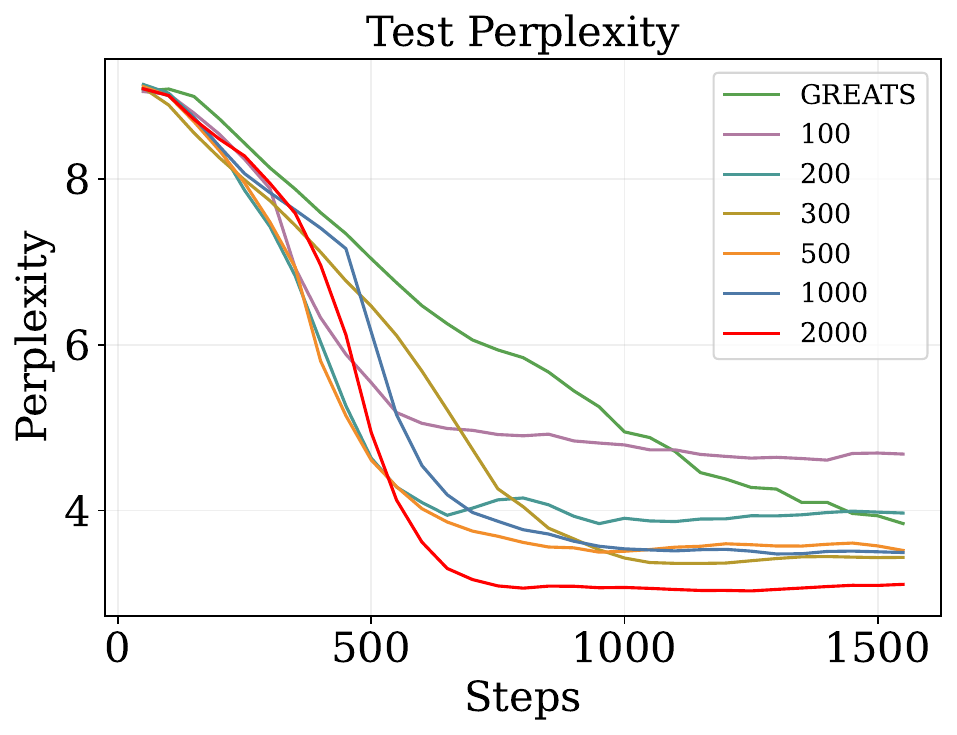}
    \caption{Impact of warm-up dataset size on test perplexity (\mmlu-Sociology). We evaluate the effect of varying the number of initial samples used to fit the Gaussian Process prior.}
\label{fig:warmup}
\end{figure}

Figure~\ref{fig:warmup} illustrates warm-up dataset size impact on our framework. Extreme undersampling (e.g., $N=100$) fails to accurately model the utility manifold, while N=200 brings a noticeable improvement. Performance plateaus across $N\in\{300,500,1000\}$, indicating that our default N=500 achieves an optimal balance between sample efficiency and computational overhead. Further scaling to $N=2000$ yields another performance leap, suggesting that a larger initial budget enables higher-resolution manifold modeling for better downstream optimization.

\end{document}